\documentclass[11pt]{article}

\usepackage[margin=1in]{geometry}
\usepackage{parskip}

\usepackage{amsmath}
\usepackage{amsfonts}
\usepackage{amsthm}
\usepackage{mathtools}
\usepackage{bm}
\usepackage{nicefrac}

\usepackage[utf8]{inputenc}
\usepackage[T1]{fontenc}

\usepackage{newtxtext}
\usepackage{newtxmath}

\usepackage{graphicx}
\usepackage{subfigure}
\usepackage{booktabs}
\usepackage{multirow}
\usepackage{stackengine}

\usepackage{natbib}
\usepackage{url}
\usepackage{xcolor}
\usepackage{microtype}
\usepackage{comment}
\usepackage{lipsum}

\usepackage[
    colorlinks=true,
    linkcolor=blue,
    citecolor=blue,
    urlcolor=blue
]{hyperref}

\usepackage[nameinlink,noabbrev]{cleveref}

\newtheorem{lemma}{Lemma}
\newtheorem{proposition}{Proposition}

\newtheorem{corollary}{Corollary}

\providecommand{\mypara}[1]{\smallskip\noindent\emph{#1} }
\providecommand{\myparab}[1]{\smallskip\noindent\textbf{#1} }

\newcommand{\mX}{\mathbf{X}}

\newcommand{\mz}{\mathbf{z}}

\newcommand{\calL}{\mathcal{L}}

\newcommand{\transpose}{\mathrm{T}}
\newcommand{\trace}{\mathrm{tr}}

\newcommand{\reals}{\mathbf{R}}

\newcommand{\rank}{\mathrm{rank}}
\newcommand{\diag}{\mathrm{diag}}

\title{On the Regularization Landscape for the Linear Recommendation Models}

\author{
\begin{tabular}{c}
Dong Li\textsuperscript{1} \quad
Zhenming Liu\textsuperscript{2} \quad
Ruoming Jin\textsuperscript{1} \quad
Hao Zhou\textsuperscript{1} \quad
Zhi Liu\textsuperscript{3} \quad
Jing Gao\textsuperscript{3} \quad
Bin Ren\textsuperscript{2}
\\[0.5em]
\textsuperscript{1} Kent State University, USA
\\
\textsuperscript{2} College of William and Mary, USA
\\
\textsuperscript{3} iLambda, USA
\\[0.3em]
\texttt{\{dli12,rjin1,hzhou6\}@kent.edu}
\\
\texttt{\{zliu,jgao\}@ilambda.com}
\\
\texttt{\{zliu,bren\}@cs.wm.edu}
\end{tabular}
}

\date{}

\begin{document}

\maketitle

\begin{abstract}
Recently, a wide range of recommendation algorithms inspired by deep learning techniques have emerged as the performance leaders on several standard recommendation benchmarks. While these algorithms were built on different DL techniques (e.g., dropouts, autoencoder), they have similar performance and even similar cost functions. This paper studies whether the models' comparable performance are sheer coincidence, or they can be unified under a single framework. We find that all linear performance leaders effectively add only a nuclear-norm based regularizer, or a Frobenius-norm based regularizer. The former ones possess a (surprising) rigid structure that limits the models' predictive power but their solutions are low rank and have closed form. The latter ones are more expressive and more efficient for recommendation but their solutions are either full-rank or require executing hard-to-tune numeric procedures such as ADMM. Along this line of finding, we further propose two low-rank, closed-form solutions, derived from carefully generalizing Frobenius-norm based regularizers. The new solutions get the best of both nuclear-norm and Frobenius-norm world.

\end{abstract}

\vspace{-.2cm}
\section{Introduction}

Research progress on algorithms for recommendation has  escalated in recent years, partially fueled by the adoption of deep learning techniques. However, recent studies have found that many new deep learning recommendation models have shown sub-par performance against the simpler linear recommendation models~\citep{RecSys19Evaluation,rendle2019difficulty}. 
Although some studies are available to analyze linear vs non-linear models~\citep{RecSys19Evaluation}, it remains puzzling why these seemingly different techniques all result in models with similar performance or even similar cost functions. This motivates us to ask a fundamental question: 

 \emph{What is the barebones engine that drives the performance improvement for the recent recommendation algorithms?} 

Specifically, do different recommendation techniques offer different ``magic'', but coincidentally have similar performance, or can they be unified under a single framework, which has a potential to produce \emph{one single algorithm} that gets \emph{the best of all existing works}? 
In the latest study, \cite{JinKDD2021} examines the relationship between the widely used matrix factorization (MF), such as ALS~\citep{warlop2017parallel}, and the linear autoencoders (LAE) which encompasses the recent performance leaders, such as EASE~\citep{Steck_2019} and EDLAE~\citep{DBLP:conf/nips/Steck20}. They consider two basic regularization forms (See Eq (1) and (6)) and found that the optimal (closed-form) solutions of both models recover the  direction of principal components, while shrinking the corresponding singular values differently. They suggests this difference may enable LAE to utilize a larger number of latent dimensions to improve recommendation accuracy, and use this to highlight the similarity as well as difference between LAE and MF. 

In this paper, we go much beyond the two basic models studied in 
~\cite{JinKDD2021} to analyze a large number of recent performance leaders of (linear) recommender algorithms. We found they all can be categorized into those that implement \emph{nuclear-norm} regularizers, and into those that implement \emph{Frobenius-norm} regularizers. We found that the former ones possess a (surprising) rigid structure that limits the models’ predictive power, and the latter ones tend to be more expressive and more effective for recommendation. In many cases, both regularizers recover the  direction of principal components, while shrinking the corresponding singular values differently. Interestingly, we observe that it is not matrix factorization or LAE that determines the shrinkage structure (as ~\cite{JinKDD2021} suggested), but instead it is the forms of regularization.  
Thus, this paper provides a more complete and accurate characterization on how a linear recommendation model performs under different regularizations.

To better understand the regularizations that can be transformed into the (weighted) nuclear-norm regularizer $\|W\|_*$,
we first show that Variational Linear AutoEncoders (VLAE) solves a weighted nuclear-norm regularization problem, in which the weights possess a specific combinatorial structure. 
We also show that this technique cannot be generalized to tackle arbitrary weight sequences. Second, it has been known that using dropout techniques is equivalent to adding a squared nuclear-norm regularizer $\|W\|^2_*$, and that the solution structure is strikingly similar to the regularizer that uses $\|W\|_*$. we generalize the result to show that the solution structures for $\|W\|^p_*$ are highly similar for all $p \geq 1$. But when $p = 1, 2$, the solution and hyper-parameters possess favorable properties so hyper-parameter search becomes easier. 
This also partially explains why only $p = 1, 2$ have been extensively considered. Third, all nuclear-norm–based techniques possess a salient property: their estimators keep the singular vectors of the data matrix and shrink only the singular values. But this severely limits the search space and explains why models that use only nuclear-norm–based regularizers share the same performance ceiling even when hyper-parameters are extensively searched. 

The (weighted) Frobenius-norm regularizers $\|\Lambda W\|^2_F$ are implemented in EASE~\citep{Steck_2019} and  EDLAE~\citep{DBLP:conf/nips/Steck20}. These models produce closed form full-rank estimators; and if the zero diagonal constraint on $W$ is enforced,  their singular vectors will no longer coincide with those of the data, and can deliver (slightly) better performance. However, no closed form solutions for the low-rank estimator is known and the current approaches rely on ADMM or stochastic factorized gradients~\citep{DBLP:conf/nips/Steck20}. 
In this paper, we propose two new low-rank, closed-form estimators that deliver comparable results to the full rank models (EASE and full-rank EDLAE) as well as the ADMM based solutions~\citep{DBLP:conf/nips/Steck20}. 

The new closed-form solutions for low-rank models have profound implications at both practical and conceptual fronts:  First, low rank solutions are often more scalable (the full rank $W$ can be too large to materialize) and can better serve real-world training and deployment. Second, perhaps more excitingly, these solutions get the best of nuclear (closed form and low rank) and Frobenius worlds (strong predictive power). A simple ``one-liner'' formula (from each solution) concisely pack all the benefits obtained by a recent long line of research and abstract out all the computation nuance (e.g., the need to tune ADMM and deal with local optimal). We believe these closed-form solutions are also powerful tools to help us compare different regularizations analytically.

\section{Background and overview}\label{sec:background}

\myparab{Recommendation system background.} Recommendation algorithms can be categorized into explicit ones that aim to predict unseen ratings between a user and an item and implicit ones that aim to predict actions, such
as user click or add-cart~\citep{Steck_2019,Dacrema_2019,zhang2019deep}. We focus on the implicit problem because it is more economically relevant. Here, let {$n$ be the number of items and $m$ be the number of users.} We are given a binary matrix $X \in \{0, 1\}^{m \times n}$ that represents the interaction between users and items so far, i.e., $X_{i,j} = 1$ iff {user $i$ has purchased or made a rating on item $j$.} Our goal is to produce a real-valued matrix $\hat X$, which we evaluate against future interactions using information retrieval metrics such as Top-$k$ or nDCG. Note that while the syntax of this problem resembles matrix completion (MC)~\citep{candes2010power}, recommendation systems and MC have different evaluation criteria so  MC's results are not directly applicable here. 

\subsection{Nuclear-norm based regularizations} 
Let $X \in \reals^{m \times n}$ be a matrix of rank at most $k$ with $k$ leading singular values being $\sigma_1(X) \geq \sigma_2(X) \geq \dots \geq \sigma_k(X)$. Let $\omega = (\omega_1, \dots, \omega_k) \in (\reals^+)^{k}$. The weighted unclear norm of $X$ with respect to $\omega$ is defined as $\|X\|_{\omega, *} = \sum_{i = 1}^k\omega_i \times \sigma_i(X)$.

We can see that this is a natural generalization of the weighted nuclear norm for low-rank matrices~\citep{gu2014weighted}. 
Also, despite its name, the weighted nuclear norm is neither convex nor differentiable unless $\omega_i$'s are sorted in \emph{descending} order~\citep{chen2013reduced,iglesias2020accurate}.

Nuclear-norm based regularizers perform $\ell_1$-shrinkage over the estimator's singular values, which resembles performing $\ell_1$-shrinkage for coefficients in a linear model in LASSO~\citep{tibshirani1996regression}. Therefore, Nuclear-norm regularizers also promote sparsity over the solution's singular values (i.e., the solution is usually low rank). We note a large fraction of recent recommendation algorithms effectively add only a nuclear-norm regularizer to MF (see also Table~\ref{tab:summary} in Appendix):  

\mypara{A1. Regularized PCA}~\citep{generalizedlowrank2016,zheng2018regularized} aims to solve
\begin{equation}
\label{eq:regulizedpca}
\min_{P, Q} \|X - PQ^T\|^2_F + \lambda\|Q\|^2_F + \lambda \|P\|^2_F.
\end{equation}
It has been known that \cref{eq:regulizedpca} is equivalent to solving $\min_{\hat X}\|X - \hat X\|^2_F + 2\lambda \|\hat X\|_*$. To solve \cref{eq:regulizedpca}, one can use factored gradient descent~\citep{bhojanapalli2016dropping} or directly use its closed form solution~\citep{pmlr-v97-kunin19a}, which involves computation of SVD of $X$. 

\mypara{A2. Matrix Factorization via dropouts.} This approach use $PQ^{\transpose}$ ($P \in \reals^{m \times k}$, $Q \in \reals^{n \times k}$) to approximate $X$ and uses a neural net to find $P$ and $Q$. A standard dropout technique is used when we train $P$ and $Q$. \citet{dropMF} shows that optimization with dropout is equivalent to solving $\min_{\hat X} \|X - \hat{X}\|^2_F + \lambda \|\hat{X}\|^2_{*}$, and the closed form solution is obtained by shrinking all singular values of $X$ by a magnitude of $\mu$, which depends on the data $X$ and choise of $\lambda$.

The next two approaches use techniques from (variational) auto-encoder. We show that they also effectively add variants of nuclear-norm regularizers although this may not be clear at the first glance.

\mypara{A3. Linear Regression via Denoising Linear Auto-Encoder} considers the following (non-uniform) weighted $\ell_2$-regularization~\citep{DBLP:conf/nips/BaoLSG20}: 
\begin{equation}
\|X- X W_1 W_2\|^2_F + \|W_1\Lambda^{\frac 1 2}\|^2_F+ \|\Lambda^{\frac 1 2} W_2\|_F^2,   
\label{eq:nonuniform}
\end{equation} 
where $\Lambda$ is a diagonal matrix. 
 \cite{DBLP:conf/nips/BaoLSG20} has shown a closed-form for eq.~\ref{eq:nonuniform} with a specific of diagonal $\Lambda=diag(\lambda_1, \lambda_2, \cdots, \lambda_k)$ when the weight is non-descending: $\lambda_1 \leq \lambda_2 \leq \cdots \leq \lambda_k$. It remains unclear if a closed-form exists for an arbitrary weight order. 

\mypara{A4. Variational Linear Auto-encoders.} While not explicitly studied before, it is also natural to consider linear simplification of  Variational Autoencoders, such as Multi-VAE ~\citep{liang2018variational}, which has shown to exhibit strong performance for recommendation. To optimize linear VAE, we need to find the MLE for the probabolistic model 
\vspace{-.1cm}
\begin{equation}
    p(\mathbf{x} \mid \mathbf{z}) =\mathcal{N}\left(W \mathbf{z}+\boldsymbol{\mu}, \sigma^{2} I\right) {\mbox{ and }}
p(\mathbf{z} \mid \mathbf{x})=\mathcal{N}(V(\mathbf{x}-\boldsymbol{\mu}), D),
\label{eq:lave-simple}
\end{equation}
where $D$ is a diagonal covariance matrix. 
Then we have the following observation.

\begin{lemma}\label{lem:nuclear}
Consider optimizing the ELBO (Evidence Lower Bound)~\citep{Kingma2014} for the above LVAE model (Eq.\ref{eq:lave-simple}). 
When the optimization is over the entire dataset and $\boldsymbol{\mu} = 0$, this optimization problem is equivalent to minimizing  
\begin{equation}\label{eq:lvae-james-2}
    \begin{split}
        \mathcal{L} &= \|X-X V^{\transpose}W^{\transpose}\|_F^2 + N\|\sqrt{D}W^T\|_F^2 + \sigma^2||X V^{\transpose}\|^2_F + g(D,\sigma)
      \end{split}  
\end{equation}
where, $g(D,\sigma)= -\sigma^2N\big(\log |D|- \trace(D) + k - n\log 2\pi\sigma^2)$.
\end{lemma}

Here, we set $\boldsymbol{\mu} = 0$ for simplicity and following typical practices in recommendations. The proof can be found in Appendix.
Since our objective is for recommendation (not purely on recovering the low rank factors of the data), we treat the covariance matrix $D$ as hyperparameters. 
Thus, the term $g(D,\sigma)$ becomes a constant, and let  $A=\sigma V^T$, $B=1/\sigma W^T$,  $\Lambda=\sigma\sqrt{N D}$. 

If we consider $D$ as an optimization parameter, then the optimal solution of LVAE is equivalent to that of pPCA ~\citep{Tipping99probabilisticprincipal}. In this case, $D=\sigma^2 (\boldsymbol{\Sigma}^2/N)^{-1}$, where $\boldsymbol{\Sigma}^2=diag(\sigma^2_i)$ are the eigen-values of the covariance matrix of $X$, and $\sigma^2 = \frac{1}{n-k}\sum_{j=k+1}^n{\frac{\sigma^2_j}{N}}$. Further, the closed form solutions of $V$ ($A$) and $W$ ($B$) are characterized. However, for recommendation, the matrix $D$ can be considered as a hyperparameter (to be learned); in this case, the closed form solution is not studied yet. We may ``clean up'' \cref{eq:lvae-james-2} and obtain the following optimization problem:

\begin{equation}
    \min_{A \in \reals^{n \times k}, B \in \reals^{k \times m}} \|X-X AB\|_F^2 + \|XA\|^2_F + \|\Lambda B\|_F^2,
\label{eqn:optclosedform}
\end{equation}
in which decision variables are $A$ and $B$, and the hyper-parameter is a diagonal matrix $\Lambda \in \reals^{k \times k}$.

\myparab{Our main results: solution structure and implications} (Sec.~\ref{section:problem}). \emph{(i)} We show that solving Eq.~\ref{eqn:optclosedform} (A4) is equivalent to solving $\|X - W\|^2_F + \lambda \|W\|_{\omega, *}$ subject to $\rank(W) \leq k$, where $\omega$ consists of $\Lambda$'s diagonal values, \emph{sorted in ascending order.} In addition, the closed form solution for $W$ is merely shrinking the $i$-th singular value of $X$ by a magnitude of $\omega_i$. When the diagonals of $\Lambda$ is already sorted (i.e., $\Lambda_{1, 1} \leq \dots \leq \Lambda_{k, k}$), the problem effectively reduces to A3. When $\Lambda$ is proportional to identity, the problem reduces to A1. This result has three major implications. First, A1, A3, and A4 effectively only add a variant of nuclear-norm based regularizer, and the major benefits from these algorithms are computational. Second, our result generalizes that in A3 and solves an open problem left there, i.e., the hyper-parameters $\Lambda$ do not need to have sorted diagonal values, the optimization algorithm will ``automatically perform the sorting''. Third, while the variational auto-encoder offers a flexibility to tailor-make the prior for each entry in the latent variable $\mz$ (in eq.~\ref{eq:lave-simple}), the solution space is quite rigid due to the auto-sorting property: the $i$-th smallest entry in $\Lambda$ will find its way to match with the $i$-th largest singular values in $X$. In other words, it is impossible to shrink $X$'s singular values by an arbitrary sequence via carefully choosing $\Lambda$ in eq.~\ref{eqn:optclosedform}. 

\emph{(ii)} We characterize the optimal solution for $\|X - W\|^2_F + \lambda \|W\|^p_*$ for any $p \geq 1$. We shall show that regardless the choice of $p$, the optimal solution uniformly shrinks all $X$'s singular values by a constant magnitude $\mu$ (and to $0$ if a singular value is already less than $\mu$). The specific $\mu$ depends on $\lambda$, $p$, as well as the data $X$ unless $p = 1$. It has two implications. First, $\mu$ needs to be fine-tuned to optimize test performance. Therefore, choices of $p$ (again) only produces computational gain. Second, when $p = 1$, $\mu$ does not depend on the data so we can tune this hyper-parameter in a direct manner. When $p = 2$, $\lambda$ is scale-invariant, i.e., it does not need to be rescaled when all entries in $X$ is scaled by a constant factor. Being able to directly tune $\mu$ or having the scale invariant property helps the hyper-parameter search; when $p \neq 1, 2$, it does not offer benefit in either computation or search, which explains why we see only $p = 1, 2$ in the literature.  

Finally, for all nuclear-norm-based approaches discussed above, the estimators always keep singular vectors of $X$ and shrink its singular values. Therefore, the solution space offered by nuclear-norm based regularization is quite constrained, which limits these models predictive power.

\subsection{Frobenius norm based regularizations}
Most algorithms below were originally motivated by the design of (denoising) auto-encoders, it has been shown that they effectively add a Frobenius-norm regularizer. See also Table~\ref{tab:summary} in Appendix. 

\mypara{A5. EASE~\citep{Steck_2019}} aims to optimize $\min_W{||X-XW||^2_F+\lambda\cdot ||W||^2_F}$ subject to the constraint that $\diag(W) = 0$. A closed form solution exists for this problem. 

\mypara{A6. DLAE~\citep{DBLP:conf/nips/Steck20}} adds a weighted Frobenius-norm regularizer so the objective becomes \[
\min_W
\left\{
\lVert X-XW\rVert_F^2
+
\lVert \Lambda^{1/2} W\rVert_F^2
\right\},
\]where $\Lambda =\frac{p}{1-p}diagM(diag(X^TX))$. 

\mypara{A7. EDLAE~\citep{DBLP:conf/nips/Steck20}} integrates weighted Frobenius norm in DLAE with EASE's diagonal constraint so its objective is the same as DLAE but it requires $\diag(W) = 0$. A closed form solution exists for this problem. When $W$ is required to be low rank, an ADMM algorithm may be used. 

\mypara{A8. Tikhonov regularization/Low Rank Regression (LRR) ~\citep{JinKDD2021}.} Let $V_k$ be the $k$ leading right singular vectors of $X$. \cite{JinKDD2021} finds an estimator that solves 
\begin{equation}
\label{eq:lowrank}
    W=\arg \min_{rank(W)\leq k} \|X -XW\|_F^2+ 
    \|\mathbf{\Gamma} W\|_F^2,
\end{equation}
where $\Gamma=\Lambda^{\frac 1 2} V_k^T $ and $\Lambda=diag(\lambda_1^\prime, \cdots, \lambda_k^\prime)$ is a hyper-parameter. 
Its a closed-form solution is. 
\begin{equation}
\label{eq:regressionlambda}
\begin{split}
W^* = V_k  diag(\frac{\sigma^2_1}{\sigma^2_1+\lambda^\prime_1},\dots,\frac{\sigma^2_k}{\sigma^2_k+\lambda^\prime_k})  V_k^T
\end{split}
\end{equation}

\myparab{Our results.} Our major goal is to design a low-rank closed form estimator whose performance is comparable to the performance leaders. We first remark that A5 and A6 produce full-rank estimators. A7 can produce either full-rank or low-rank estimator (via ADMM) and has the best performance (among all approaches we discussed). The estimator from A8 is low-rank and has a closed-form solution but it has to keep singular vectors of $X$ so its predictive power is also limited. Nevertheless, A8 is conceptually interesting because it uses Frobenius norm regularizers but its solution space cover the solution space offered in A4 (and thus also A1-A4). 

\begin{proposition}
\label{Tikhonovpower}
For any regularized instances in the form of (\ref{eqn:optclosedform}) with regularization parameter $\Lambda$ such that $\sigma_i(X) \geq \lambda_{(k - i)}$ for all $i$, there is a corresponding Tikhonov regularized instance with $\mathbf{\Gamma}=\Lambda^{\frac 1 2} V_k^{\transpose}$ which provides the same regularization effect. 
\end{proposition}

The proof is in Appendix. A major implication of Prop.~\ref{Tikhonovpower} is that we can focus on designing Frobenius-norm regularizers because it also gets the value from using nuclear-norm regularizers. Indeed, Sec.~\ref{section:Tikhonov} will introduce two low-rank Frobenius-norm-based model with closed form solutions that have comparable performance to linear performance leaders.

\vspace{-.1cm}
\section{Nuclear-norm based regularization}
\label{section:problem}
\vspace{-.1cm}
\myparab{Rigidity of VLAE.} We first analyze solution for Eq.~\ref{eqn:optclosedform} (A4). To facilitate the analysis, we also consider the following problem:

\begin{equation}
\min_{P, Q} \|X - PQ\|^2_F + \|\Lambda^{\frac 1 2} Q\|^2_F + \|P\ \Lambda^{\frac 1 2}|^2_F, \text{ or equivalently }, 
\min_{P, Q} \|X - PQ\|^2_F + \|\Lambda Q\|^2_F + \|P\|^2_F,
\label{eq:gmf}
\end{equation}

Note that when we let $Q^\prime = X A^*$ and $P^\prime = B^*$ (where $A^*$ and $B^*$ are an optimal solution of  Equation~\ref{eqn:optclosedform}), the syntax of Eq.~\ref{eq:gmf} matches with that of Eq.~\ref{eqn:optclosedform}. This implies that solution for Eq.~\ref{eq:gmf} is a lower bound of that for Eq.~\ref{eqn:optclosedform}. These two solutions coincide only when the columns in the optimal $P^*$ in Eq.~\ref{eq:gmf} are spanned by the columns of $X$. 

We shall first find a closed-form solution $(P^*, Q^*)$ for Eq.~\ref{eq:gmf}, and show that indeed that the column space of $P^*$ is in the column space of $X$. Below is our major Proposition.

\begin{proposition}\label{prop:diagonalcost} Let $f: \reals^{m \times n}\rightarrow \reals^+$ be any cost function. Let $P \in \reals^{m \times k}$ and $Q \in \reals^{k \times n}$. Let $\Lambda \in \reals^{k \times k}$ be a diagonal matrix such that $\lambda_i = \Lambda_{ii} \geq 0$ ($i \in [k]$). Let also $\omega = (\lambda_{\pi(1)}, \lambda_{\pi(2)}, \dots , \lambda_{\pi(k)})$, where $\pi$ is a permutation on $[k]$ such that $\lambda_{\pi(1)}\leq  \lambda_{\pi(2)} \leq \dots \leq \lambda_{\pi(k)}$. The following two optimization problems have the same optimal values
\begin{align*}
\begin{array}{lll}
    OPT1: & \min_{P, Q} & f(PQ) + \|\Lambda^{\frac 1 2}Q\|^2_F + \|P \Lambda^{\frac 1 2}\|^2_F. \\
    \\
    OPT2: & \min_W &  f(W) + 2\|W\|_{\omega, *} \\
    & {\mbox{subject to}} & \rank(W) \leq k. 
\end{array}
\end{align*}
In addition, if $(P^*, Q^*)$ is an optimal solution for $OPT1$, then $W^{*} = P^*Q^*$ is an optimal solution for $OPT2$. If $W^*$ is an optimal solution for $OPT2$, then there exists an optimal solution $(P^*, Q^*)$ for $OPT1$ such that $W^* = P^*Q^*$.
\end{proposition}

We reiterate three points made earlier (Sec.~\ref{sec:background}). \emph{(i)} Both A3 and A4 effectively add a nuclear norm regularizer. \emph{(ii)} Diagonals of $\Lambda$ do not need to be sorted in ascending order as stated in~\citep{DBLP:conf/nips/BaoLSG20} because any permutation of the diagonals will be equivalent to $OPT2$. This also limits the search space and affects a model's prediction power. \emph{(iii)} Prop.~\ref{prop:diagonalcost} ``compiles'' a non-differentiable objective ($OPT2$) into an equivalent differentiable one ($OPT1$), which is easier to optimize. In addition, $f(\cdot)$ in A3 \& A4 is  the reconstruction error, in which case closed form solutions exist.

We next explain the intuition for proving 
Prop.~\ref{prop:diagonalcost} (see Appendix for the full analysis). Consider $OPT1$ and 
let $W = PQ$. Our goal is to characterize the behaviors of $P$ and $Q$ with the presence of the regularizers when $W = PQ$ is known (fixed). Let the SVD of $W$ be $U_W \Sigma_W V^{\transpose}_{W}$. Because two regularizers $\|\Lambda^{\frac 1 2}Q\|^2_F$ and $\|P\Lambda^{\frac 1 2}\|^2_F$ are symmetric, we could ``guess'' $P = U_W \Sigma^{\frac 1 2}_W \Omega$ and $Q = \Omega^{\transpose}\Sigma^{\frac 1 2}_W V^{\transpose}_W$, where $\Omega$ is a unitary matrix. Now we have \vspace{-.1cm}
\begin{align*}
\|\Lambda^{\frac 1 2}Q\|^2_F + \|P \Lambda^{\frac 1 2}\|^2_F = \|\Lambda^{\frac 1 2}\Omega \Sigma^{\frac 1 2}V^{\transpose}_W\|^2_F + \|U_W \Sigma^{\frac 1 2}_W \Omega \Lambda^{\frac 1 2}\|^2_F = 2 \|\Lambda^{\frac 1 2}\Omega \Sigma^{\frac 1 2}_W\|^2_F. 
\end{align*}
\vspace{-.1cm}
Now the question of finding $P$ and $Q$ when $W$ is known boils down to finding a unitary matrix $\Omega$ that minimizes $\|\Lambda^{\frac 1 2}\Omega \Sigma^{\frac 1 2}_W\|^2_F$, where diagonal matrices $\Lambda$ and $\Sigma_W$ are given. Recall that $\lambda_i = \Lambda_{ii}$ and let $\sigma_i = (\Sigma_W)_{ii}$. Note that $\lambda_i$'s could be unsorted and, and that $\sigma_i$'s are sorted in descending order. 

If we restrict $\Omega$ to be only a permutation matrix, then we aim to find a permutation $\pi \in [k]$ that minimizes $\sum_{i \leq k}\lambda_{\pi(i)}\sigma_i$. Using a rearrangement inequality~\cite{yue2020matrix}, we can see that the minimal is achieved when $\lambda_{\pi(1)} \leq \lambda_{\pi(2)} \leq \dots \leq \lambda_{\pi(k)}$. In this case, we indeed have $\min_{\Omega \mbox{ \small a permutation}} \|\Lambda^{\frac 1 2}\Omega \Sigma^{\frac 1 2}_W\|^2_F = \|\Sigma_W\|_{\omega, *}$, 
where $\omega = (\lambda_{\pi(1)}, \dots, \lambda_{\pi(k)})$. 

Note that because $PQ = P \Omega \Omega^{\transpose}Q$ for any unitary matrix $\Omega$, it is always beneficial to use $\Omega$ to shuffle the rows and columns of $P$ and $Q$ so that the largest $\sigma_i$ is mapped to the smallest $\lambda_i$, etc. This ``degree of freedom'' from $\Omega$ also explains why ordering the values along $\Lambda$'s diagonal is irrelevant. 

Appendix shows that even when $\Omega$ is allowed to be any unitary matrix, the optimal one is still a permutation matrix. This conclusion can be viewed as a matrix version of re-arrangement inequality. 

Prop.~\ref{prop:diagonalcost} also leads to the following Corollary.  

\begin{corollary}\label{cor:closedform}
Let $\mX \in \reals^{m \times n}$ ($m \geq n$) be a full rank matrix. Let $\Lambda$ be a diagonal matrix with $\Lambda_{ii} \geq 0$ for all $i$. Let $P \in \reals^{m \times k}$ and $Q \in \reals^{k \times n}$. Consider the optimization problems
\begin{align}
    \min_{P, Q}\|X - PQ\|^2_F + \|\Lambda^{\frac 1 2}Q\|^2_F + \|P \Lambda^{\frac 1 2}\|^2_F   \label{eqn:factorlambda} \\
\min_{A, B}\|X - X AB\|^2_F + \|\Lambda B\|^2_F + \|X A\|^2_F. \label{eqn:factorlambdaab}
\end{align}
Let the SVD of $X$ be $U\Sigma V^{\transpose}$, and let $\Sigma_k \in \reals^{k \times k}$ be a matrix comprising the $k$ largest singular values of $X$, and $U_k$ and $V_k$ be the corresponding singular vectors. Let $\lambda_{(1)} \geq \lambda_{(2)} \geq \dots \geq \lambda_{(k)}$ be the sorted sequence of the diagonal values from $\Lambda$. (\ref{eqn:factorlambda}) has a closed-form solution:
{\small
\begin{equation}
\label{eq:mfsolution}
\begin{split}
    P^* & = U_k diag(\sqrt{(\sigma_1 - \lambda_{(k)})^+}), \dots, \sqrt{(\sigma_k - \lambda_{(1)})^+}) \Omega,\\ 
      Q^* & = \Omega^{\transpose}diag(\sqrt{(\sigma_1 - \lambda_{(k)})^+}), \dots) V^{\transpose}_k, 
     \end{split}
\end{equation}}
where $\Omega$ is a unitary matrix that corresponds to the permutation $\pi$ such that $\lambda_{\pi(1)} \leq \dots \leq \lambda_{\pi(k)}$. In addition, (\ref{eqn:factorlambdaab}) has a closed-form solution: 
$A^* = X^{\dagger}P^* \Lambda^{\frac 1 2} \mbox{ and } B^* = \Lambda^{- \frac 1 2}Q^*$, 
where $X^{\dagger}$ is the pseudo-inverse of $X$. 
\end{corollary}

\myparab{Uniform solution structure for regularizing $\|W\|^p_*$.} In~\citep{dropMF}, it was observed that when we use a neural net $X = PQ^{\transpose}$ (with learnable parameters being $P$ and $Q$) to train a model and a standard dropout is used, the objective is equivalent to solving $\min_W \|X - W\|^2_F + \lambda \|W\|^2_*$. While the regularizer $\|W\|^2_*$ deviates from the standard one $\|W\|_*$, the optimal solution here is $W = U S_{\mu}(\Sigma)V^{\transpose}$, where $U\Sigma V^{\transpose}$ is SVD of $X$, and $S_{\mu}(\Sigma)$ is a diagonal matrix such that its $(i,i)$-th element is $(\Sigma_{i,i} - \mu)^+$, in which $\mu$ depends on the data $X$ and $\lambda$. In other words, the optimal solution for regularizers $\|W\|^2_*$ and $\|W\|_*$ are strikingly similar. Thus, we are interested in how regularizers with different exponents are connected. Our main observation is that for \emph{any} regularizer $\|W\|^p_*$ ($p \geq 1$), the optimal solution has the same structure. 

\begin{lemma}Let $X \in \reals^{m \times n}$, where $m \geq n$. Let the $d$ leading SVDs of $X$ be $U_d$ and $V_d$ respectively. Let $\sigma_1, \dots, \sigma_n$ be the singular values of $X$. Consider the optimization problem: 
\begin{equation}
    \min_W \frac 1 2 \|X - W\|^2_F + \lambda \|W\|^p_*.
\label{eqn:exponent}
\end{equation}
Let $\mu_k = \frac 1 k \sum_{i \leq k}\sigma_i(X)$, $\eta(\mu_k)$ be the positive root of the function $z + \lambda k z^{p - 1}- k \mu_k$ and $d$ be the largest value such that $\sigma_d(X) - \lambda (\eta(\mu_d))^{p - 1}\geq 0$. 
Let $\mu = \lambda (\eta(\mu_d))^{p - 1}$. 
The optimal solution of $W$ is $U_d \diag(\sigma_1 - \mu, \sigma_2 - \mu, \dots, \sigma_d - \mu) V^{\transpose}_d$. 
\label{lem:exponent}
\end{lemma}

Lemma~\ref{lem:exponent} is a straightforward generalization of Prop.~12 in~\citep{dropMF} so the contribution here is a conceptual one: it implies that regularizers $\|W\|^p_*$ with different $p$'s differ in how the shrinkage variable $\mu$ is obtained. Observe that $\mu$ is an important hyper-parameter that needs to be extensively tuned against data, all regularizers $\|W\|^p_*$ provide the same learning power. As noted earlier, $\mu$ is a function of $X$ (i.e., different $X$ will result in different $\mu$) unless $p = 1$. In addition, $\lambda$ needs to be rescaled when $X$ is scaled by a constant factor unless $p = 2$. This implies it could be easier to tune $\mu$ when $p = 1, 2$, and explains why only $p = 1, 2$ have been extensively considered.

\vspace{-.2cm}
\section{Low-Rank Frobenius norm based regularizations}
\label{section:Tikhonov}
\vspace{-.1cm}
This section presents the close-form low-rank estimators with comparable performance to the state of the art algorithms.

\myparab{Approximate Low rank DLAE and EDLAE.} 
Recall that  DLAE solves
 \begin{equation}\label{eq:full_edlae}
     \begin{split}
         & \min_{W}||X-XW||^2_F+||\Lambda^{\frac{1}{2}}W||^2_F\\
         s.t.& \quad \Lambda=\frac{p}{1-p} dMat(diag(X^T X)),
     \end{split}
 \end{equation}
 and EDLAE with the additional $diag(W)=0$ constraint. 
 Even though both the Nuclear norm based regularization as well as the full rank DLAE and EDLAE solutions all have closed-form solutions, such solution is unknown for the low-rank DLAE and EDLAE, whose existing solution is based on ADMM\citet{DBLP:conf/nips/Steck20}. 
 The closed form solutions will help both better understand and compare these models, and determine the hyper-parameters, which is usually difficult for the ADMM type solutions. 
 
 For DLAE, since it can be considered a special form of Tikhonov regularization (Eq 6) which has a closed form~\cite{JinKDD2021}, its closed form solution, referred to LR-DLAE, is immediately available (See Line $6$ in Table ~\ref{tab:summary}). 
 However, for EDLAE, it has the zero diagonal constraint, which make the exact solution difficult to express. 
 Here, we present an approximate low-rank closed-form solution of \cref{eq:full_edlae} by decomposing the optimization problem into two subproblems, which is similar to \citep{JinKDD2021}:
We first consider the full-rank closed form solution for EDLAE ~\cite{DBLP:conf/nips/Steck20}, which is: 
\begin{equation*}
\begin{split}
       W^*&= I - C\cdot dMat(1\oslash diag(C)), \text{ where }C =(X^TX+\Lambda)^{-1}\\
\end{split}
\end{equation*}
Then, we consider two approaches to produce low-rank matrix approximation of $W^*$: 

\noindent{\bf (Method 1 (LR-EDLAE-1): ) Selecting $\widehat{W}$ to best approximate the performance of $W^*$ without the zero diagonal constraint:}
\begin{equation}\label{eq:subp2}
\begin{split}
 \widehat{W} & =\arg\min_{rank(W) \leq k} ||\overline{X}W^*-\overline{X}W||_F^2  \\
   &= \arg \min_{rank(W) \leq k} ||X W^* - XW||_F^2+ 
    ||\Lambda^{\frac{1}{2}} (W^*- W)||_F^2
\end{split}
\end{equation}
where $\overline{X}=\begin{bmatrix} X  \\ {\Lambda^{\frac{1}{2}}}   \end{bmatrix}$. Noting, in the full rank problem \cref{eq:full_edlae}, it forces the diagonal of derived matrix ($W^*$) to be zero. Here, we relax the zero diagonal constraint - the diagonal of low rank approximate matrix $\widehat{W}$ doesn't have to be zero, which has also been discussed in \citep{DBLP:conf/nips/Steck20}. The closed-from solution of \cref{eq:subp2} is given by:
\begin{equation} 
\boxed{
\widehat{W}= W^*(Q_k Q_k^T)
}
\end{equation}
where $Q_k$ comes from SVD:
\begin{equation*}
    \begin{split}
        \overline{Y}^*&=\overline{X}W^*, \text{ and } \overline{Y}^*(k) =P_k \Sigma_k Q^T_k
    \end{split}
\end{equation*}

\noindent{\bf (Method 2 (LR-EDLAE-2):) SVD approximation of $W^*$:} The alternative solution is to simply perform SVD, and which gives a low-rank estimation of $W^*$. 

Note that in both approaches, the hard constraint of zero-diagonal on $\widehat{W}$ is relaxed. In the next section, the experimental results show both approaches can provide comparable or better performance compared with the ADMM solution, and also very close to the full rank EDLAE solution.

\vspace*{-1.0ex}
\section{Experimental Results}
\label{section:experiments}

In this section, we experimentally study different regularizations for linear recommendation models.  
Our goal is to validate the effectiveness of various regularizations (all can be categorized under nuclear norm and Frobenius norm) together with their simple closed-form solutions. 
We aim to answer three questions: \textbf{Q1.} How does the closed form solution of low rank Frobenius norm perform compared with the ADMM solutions (Section 4) and how does the weighted nuclear norm regularizer for matrix factorization (Proposition~\ref{prop:diagonalcost} and Corollary~\ref{cor:closedform}) perform?  \textbf{Q2.} What is the tradeoff between the number of factors (rank $k$) and the recommendation accuracy?  \textbf{Q3.} How does the ordering of weights from small to large (non-descending) for adjusting the singular values (Corollary~\ref{cor:closedform}) affect the recommendation performance?

\noindent{\bf Experimental Setup:}
We use  three  commonly used datasets for recommendation studies: MovieLens 20 Million (ML-20M) \citep{ml20m}, Netflix Prize (Netflix) \citep{netflix}, and the Million Song Data (MSD)\citep{msddataset}.  We obtained these datasets and all benchmarks from authors of EASE~\citep{Steck_2019}, EDLAE~\citep{DBLP:conf/nips/Steck20},  and Mult-VAE~\citep{liang2018variational}. 

Similar to the latest study in EASE~\citep{Steck_2019}, and  EDLAE~\citep{DBLP:conf/nips/Steck20}, we consider the following state-of-the-art recommendation models: ALS (WMF)~\citep{hu2008collaborative} for matrix factorization approaches, SLIM~\citep{slim01}, EASE~\citep{Steck_2019}, and EDLAE~\citep{DBLP:conf/nips/Steck20} for linear autoendoers, CDAE ~\citep{cdae16}, Mult-DAE and  Mult-VAE~\citep{liang2018variational} for deep learning models.  The experiment settings for these baseline are the same as  ~\citep{liang2018variational,Steck_2019,DBLP:conf/nips/Steck20}. Also we follow their practice ~\citep{liang2018variational,Steck_2019,DBLP:conf/nips/Steck20} for the {\em strong generalization} by splitting the users into training, validation and tests group, and report performance metrics $Recall@20$, $Recall@50$ and $nDCG@100$. 
Finally, we note that our code are openly available (see Appendix).

\noindent{\bf Q1: Low-Rank Frobenius Norm  and (Weighted) Nuclear Norm Regularization:} 
In this experiment, we evaluate the low rank Frobenius norm regularization and the nuclear norm regularization (\cref{eq:regulizedpca}) for the matrix factorization.
Here EDLAE-ADMM, LRR, LR-DLAE, LR-EDLAE-1, LR-EDLAE-2, MF dropout and LVAE are listed in \cref{tab:summary} in Appendix.
To determine the non-descending order of weights $\lambda_i$ for the closed-form solution in (\cref{eq:mfsolution}), we follow the practice in weighted nuclear norm regularization in~\citep{gu2014weighted} as well as the optimized pPCA weight~\citep{ppca_elbo}. Let $\lambda_i=\frac{C}{\sigma_i}$ where $C$ is a hyperparameter, and we perform grid-search to find the optimal one.

\begin{table}[!b]
\caption{The performance comparison between different regularizations. For notation, please refer \cref{tab:summary} for more details.}
\label{tab:main-table}
\resizebox{\textwidth}{!}{%
\begin{tabular}{|c|c|l|c|c|c|c|c|c|c|c|c|}
\hline
\multicolumn{3}{|c|}{\multirow{2}{*}{Model}}                       & \multicolumn{3}{c|}{ML-20M} & \multicolumn{3}{c|}{Netflix} & \multicolumn{3}{c|}{MSD} \\ \cline{4-12} 
\multicolumn{3}{|c|}{} &
  Recall@20 &
  Recall@50 &
  nDCG@100 &
  Recall@20 &
  Recall@50 &
  nDCG@100 &
  Recall@20 &
  Recall@50 &
  nDCG@100 \\ \hline
\multirow{8}{*}{Frobinius Norm} &
  \multicolumn{2}{c|}{EASE} &
  0.391 &
  0.521 &
  0.420 &
  0.362 &
  0.445 &
  0.393 &
  0.333 &
  0.428 &
  0.389 \\ \cline{2-12} 
 &
  \multicolumn{2}{c|}{DLAE} &
  0.392 &
  0.527 &
  0.424 &
  0.362 &
  0.446 &
  0.395 &
  0.329 &
  0.426 &
  0.387 \\ \cline{2-12} 
 &
  \multicolumn{2}{c|}{EDLAE} &
  0.393 &
  0.523 &
  0.424 &
  0.366 &
  0.449 &
  0.398 &
  0.334 &
  0.429 &
  0.392 \\ \cline{2-12} 
 &
  \multicolumn{2}{c|}{EDLAE-ADMM} &
  0.392 &
  0.524 &
  0.424 &
  0.365 &
  0.448 &
  0.396 &
  0.330 &
  0.424 &
  0.386 \\ \cline{2-12} 
 &
  \multicolumn{2}{c|}{LRR} &
  0.376 &
  0.511 &
  0.408 &
  0.348 &
  0.431 &
  0.380 &
  0.248 &
  0.335 &
  0.301 \\ \cline{2-12} 
 &
  \multicolumn{2}{c|}{LR DLAE} &
  0.392 &
  0.527 &
  0.424 &
  0.362 &
  0.445 &
  0.395 &
  0.306 &
  0.403 &
  0.363 \\ \cline{2-12} 
 &
  \multicolumn{2}{c|}{LR-EDLAE-1} &
  0.392 &
  0.523 &
  0.424 &
  0.365 &
  0.449 &
  0.398 &
  0.327 &
  0.423 &
  0.384 \\ \cline{2-12} 
 &
  \multicolumn{2}{c|}{LR-EDLAE-2} &
  0.392 &
  0.523 &
  0.424 &
  0.365 &
  0.449 &
  0.398 &
  0.325 &
  0.421 &
  0.382 \\ \hline
\multirow{3}{*}{Nuclear Norm}    & \multicolumn{2}{c|}{MF dropout} & 0.367   & 0.501   & 0.393   & 0.334    & 0.418   & 0.365   & 0.270  & 0.367  & 0.328  \\ \cline{2-12} 
 &
  \multicolumn{2}{c|}{Regularized PCA} &
  0.364 &
  0.501 &
  0.392 &
  0.331 &
  0.417 &
  0.365 &
  0.229 &
  0.313 &
  0.279 \\ \cline{2-12} 
 &
  \multicolumn{2}{c|}{LVAE} &
  0.348 &
  0.474 &
  0.378 &
  0.325 &
  0.405 &
  0.357 &
  0.205 &
  0.254 &
  0.286 \\ \hline
\multirow{5}{*}{Baseline} &
  \multicolumn{2}{c|}{WMF/ALS} &
  0.360 &
  0.498 &
  0.386 &
  0.316 &
  0.404 &
  0.351 &
  0.211 &
  0.312 &
  0.257 \\ \cline{2-12} 
 &
  \multicolumn{2}{c|}{SLIM} &
  0.370 &
  0.495 &
  0.401 &
  0.347 &
  0.428 &
  0.379 &
  \multicolumn{3}{c|}{no results in \citep{slim01}} \\ \cline{2-12} 
 &
  \multicolumn{2}{c|}{CDAE} &
  0.391 &
  0.523 &
  0.418 &
  0.343 &
  0.428 &
  0.376 &
  0.188 &
  0.283 &
  0.237 \\ \cline{2-12} 
 &
  \multicolumn{2}{c|}{MULT-DAE} &
  0.387 &
  0.524 &
  0.419 &
  0.344 &
  0.438 &
  0.380 &
  0.266 &
  0.363 &
  0.313 \\ \cline{2-12} 
 &
  \multicolumn{2}{c|}{MULT-VAE} &
  0.395 &
  0.537 &
  0.426 &
  0.351 &
  0.444 &
  0.386 &
  0.266 &
  0.364 &
  0.316 \\ \hline
\multicolumn{3}{|c|}{\# items} &
  \multicolumn{3}{c|}{20108} &
  \multicolumn{3}{c|}{17769} &
  \multicolumn{3}{c|}{41140} \\ \hline
\multicolumn{3}{|c|}{\# users} &
  \multicolumn{3}{c|}{136677} &
  \multicolumn{3}{c|}{463435} &
  \multicolumn{3}{c|}{571353} \\ \hline
\multicolumn{3}{|c|}{\# interactions} &
  \multicolumn{3}{c|}{10mil} &
  \multicolumn{3}{c|}{57mil} &
  \multicolumn{3}{c|}{34mil} \\ \hline
\end{tabular}%
}
\end{table}

In Table~\ref{tab:main-table}, we can see that the weighted nuclear norm regularization (LVAE) based matrix factorization actually performs worse than the constant weighted version (Regularized PCA). And the latter shows very strong performance comparing against the WFM/ALS (one of the most popular implicit matrix factorization algorithm). We also observe the closed form solutions (LR-DLAE, LR-EDLAE-1 and LR-EDLAE-2) all perform very comparable with the ADMM based low rank solution and the full rank DLAE and EDLAE solutions.

\begin{figure}
    \centering
    \subfigure[]{\includegraphics[width=0.32\textwidth]{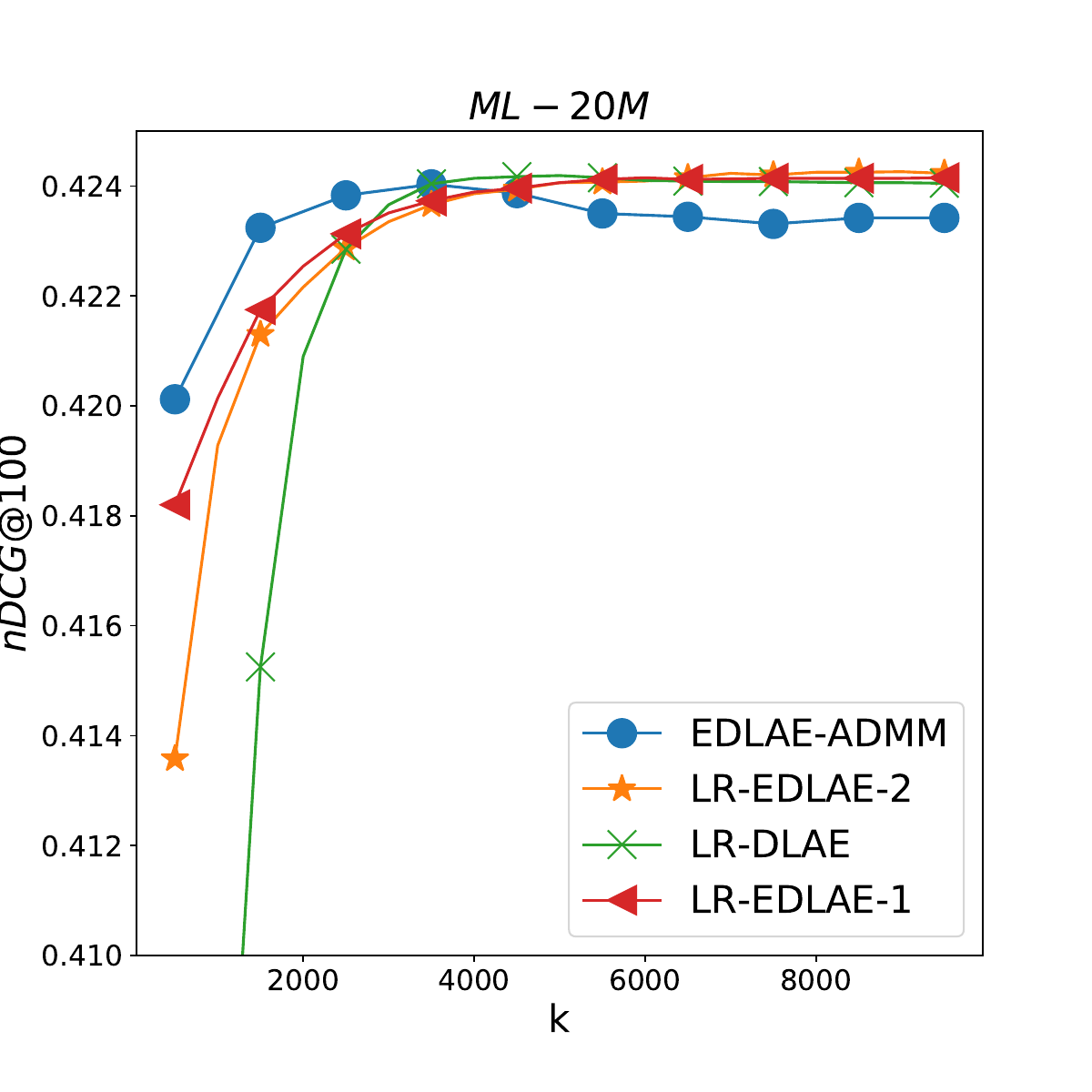}} 
    \subfigure[]{\includegraphics[width=0.32\textwidth]{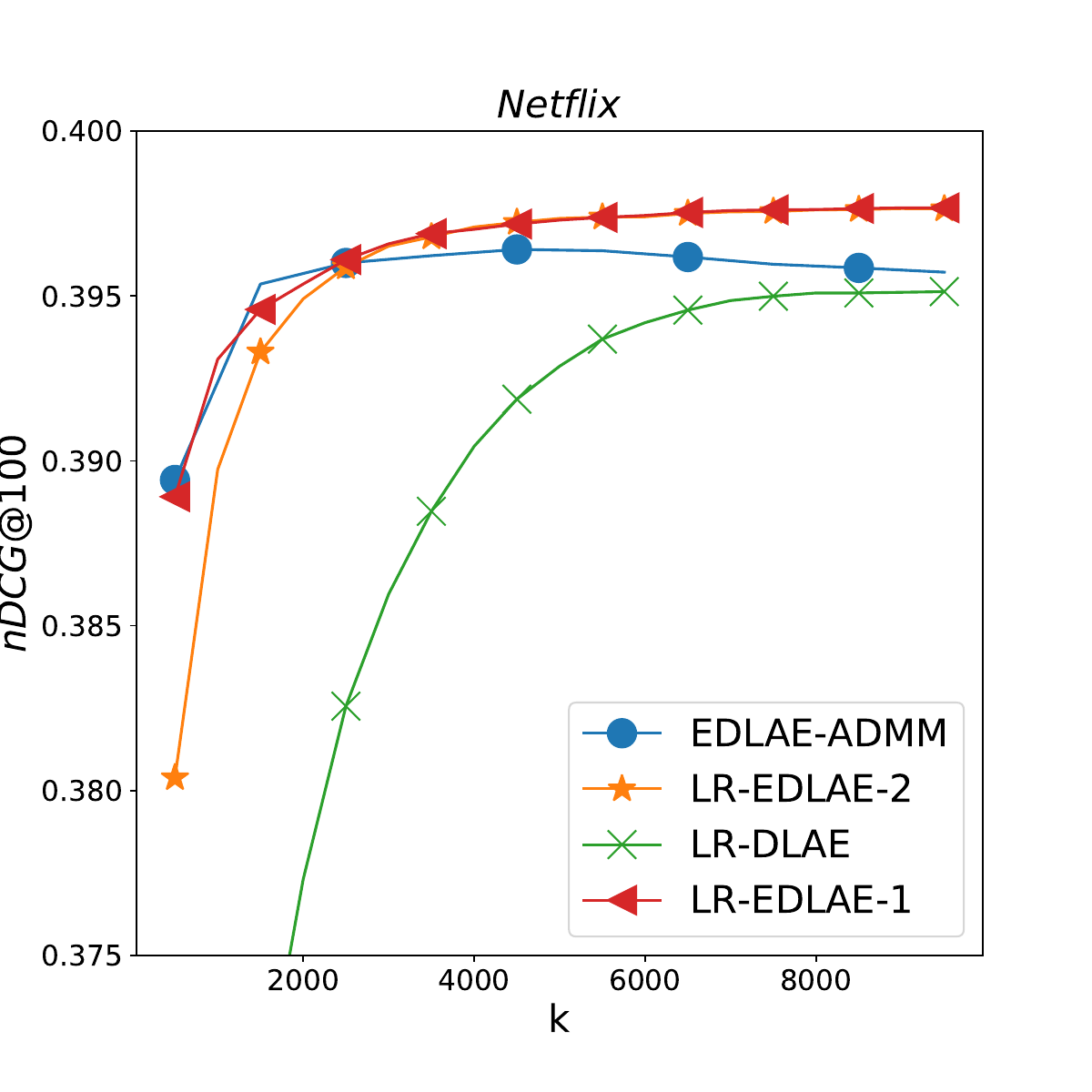}} 
    \subfigure[]{\includegraphics[width=0.32\textwidth]{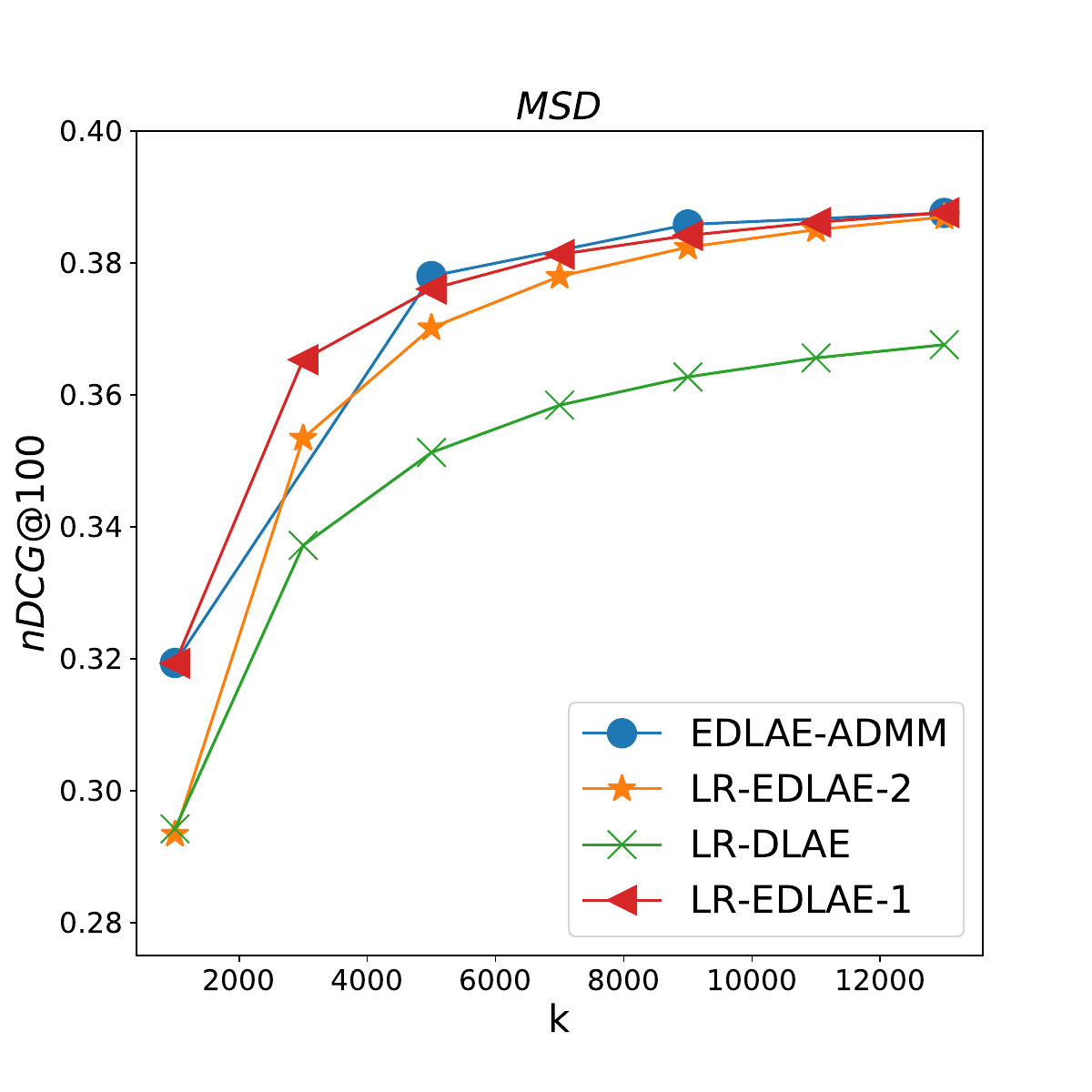}}
    \caption{Low rank models $nDCG@100$ on test data for 3 datasets.}
    \label{fig:foobar}
\end{figure}

\noindent{\bf Q2: nDCG vs Rank $k$ for low-rank Frobenius norm:}
In this experiment, we focus on evaluating the recommendation accuracy (using nDCG) against the rank $k$. Specifically, we vary the rank $k$ from around $1K$ to around $10K$, and we tune and compare four different methods, including EDLAE-ADMM, LR-DLAE, LR-EDLAE-1, and LR-EDLAE-2. We have the following observations: 1) As $k$ increases, the recommendation accuracy also increases in general; however, most of them reaches a plateau around similar $K$, and for different datasets, the saturating point varies. 2) LR-DLAE performs worse than EDLAE based approaches in two out of three datasets. This partially demonstrates the benefits of zero-diagonal constraint. 3) The closed form solution of EDLAE performs comparable or even slightly better than ADMM methods as $K$ grows; but when $K$ is relatively small, ADMM method perform slightly better. But none-the-less, for most of the reasonable choices of $k$ when low-rank approximates full rank, the closed form solution performs comparable or better.

\noindent{\bf Q3: Impact of Weight Ordering:}
Finally, we study how the ordering of weights from small to large (non-descending) for adjusting the singular values (Corollary~\ref{cor:closedform}) affects the recommendation performance using matrix factorization (closed-form solution in Eq.~\ref{eq:mfsolution}). Our results are in Table~\ref{tab:mf-table}. Here, we obtain the searched optimal weight parameters from weighted Tikhonov regularization (following the approach in ~\cite{JinKDD2021}), and map it back to the parameters in the closed form solution (Proposition~\ref{Tikhonovpower}). Then we sort the parameters in the non-descending order, and then report their results in the second row of Table~\ref{tab:mf-table}. We can see that the recommendation  performance becomes significant worst. This help confirm our conjecture that the strict ordering of weight on matrix factorization and other regularizations can be an inherent limitation for those approaches.

\begin{table}[]
\caption{Investigating the weight ordering of Matrix Factorization}
\label{tab:mf-table}
\resizebox{\textwidth}{!}{%
\begin{tabular}{|c|l|c|c|c|c|c|c|c|c|c|}
\hline
\multicolumn{2}{|c|}{\multirow{2}{*}{Model}} & \multicolumn{3}{c|}{ML-20M}      & \multicolumn{3}{c|}{Netflix}     & \multicolumn{3}{c|}{MSD}         \\ \cline{3-11} 
\multicolumn{2}{|c|}{}                       & Recall@20 & Recall@50 & nDCG@100 & Recall@20 & Recall@50 & nDCG@100 & Recall@20 & Recall@50 & nDCG@100 \\ \hline
\multicolumn{2}{|c|}{MF/LRR weighted} & 0.3806 & 0.5175 & 0.4102 & 0.3484 & 0.4320 & 0.3797 & 0.2508 & 0.3390 & 0.3037 \\ \hline
\multicolumn{2}{|c|}{MF sorted}     & 0.3017 & 0.4507 & 0.3361 & 0.2860 & 0.3801 & 0.3265 & 0.2288 & 0.3148 & 0.2802 \\ \hline
\end{tabular}%
}
\end{table}

\section{Conclusion and Discussion}
This work provides a complete analysis on the recently proposed linear models for recommendation systems. Despite that models leverage different deep learning techniques, they achieve similar performance. We find that this is not coincident: all the models add either a nuclear-norm-based (Lemma~\ref{lem:nuclear}) or a Frobenius-norm based regularizer. The nuclear-norm-based approach results in estimators that keep $X$'s singular vectors  and shrink its singular values in a quite rigid way (Proposition~\ref{prop:diagonalcost} and Lemma~\ref{lem:exponent}), which limit their prediction power. The Frobenius-norm models are more express (Proposition~\ref{Tikhonovpower}) and effective but their estimators are either full-rank or do not have closed form solutions. To get the best of both nuclear and Frobenius worlds, 
we propose two low-rank and closed-form estimators (Sec.~\ref{section:Tikhonov}) based on carefully generalizing Frobenius-norm based regularizers. These estimators have competitive performance against linear performance leaders, and thus concisely pack all the benefits obtained by a recent long line of research and abstract out all the computation nuance.

\newpage

\bibliographystyle{plainnat}
\bibliography{ref}

\newpage

\appendix

\vspace*{-3.0ex}
\section{Related Work}\label{problem}
\vspace*{-2.0ex}
There have been extensive researches on recommendation~\citep{charubook}. Besides the basic user-based and item-based collaborative filtering~\citep{DeshpandeK@itemKNN}, the full rank linear autoencoder approaches include SLIM \citep{slim01}, HOLISM ~\citep{hoslim14}, EASE~\citep{Steck_2019}, DLAE (Denoising linear autoencoder)~\cite{DBLP:conf/nips/Steck20}, whereas low-rank approaches include~\citep{fism13,LRec16,DBLP:conf/nips/Steck20}. All the customized recommendation has been enforcing zero diagonal constraints for generalization purpose, whereas we show an approximate closed-form solution for a two-term Tikhonov regularization without the zero diagonal constraint can be as effective as these models.

Matrix factorization has been been widely studied in practice, partially due to Netflix competition~\cite{mfsurvey}.  Methods like  SVD++ ~\cite{Koren08} and implicit Alternating Least Square (ALS) method~\cite{hu2008collaborative} (also weighted matrix factorization) have been very influential.  \citep{JinKDD2021} shows the relationship between linear autoencoders and matrix factorization, and pointed out a potential advantage of linear autoencoders. In this work, we take a step further to reveal a deeper relationship between Tikhonov regularized linear autoencoders and a few other regularizations including matrix factorization, and show the potential limitation of the class of regularization.  
We also utilize the linear variational autoencoders (LVAE) to study how the deep VAE based recommendation approaches ~\citep{cdaekdd17,liang2018variational,recvaewsdm20} relate to linear autoencoders and matrix factorization.

Outside recommendation, there have been a few recent studies on regularization landscapes of linear (variational) autoencoders~\citep{pmlr-v97-kunin19a,DBLP:conf/nips/BaoLSG20,NEURIPS2019_7e3315fe}. They do not provide the general weighted $\ell_2$ regularization and thus did not find the inherent limitation on the regularization (for MF). Our LVAE inspired regularization is also never studied before. 

Nuclear norm regularizers can recover low-rank matrices in the vector regression setting~\citep{negahban2011estimation}. Its weighted generalization can be applied in the area of image processing~\citep{gu2014weighted}. Because weighted nuclear-norm is usually not convex or differentiable, finding optimal solutions is difficult except for  a few special cases ~\citep{chen2013reduced}.

\begin{table}[th]
\caption{Investigating the closed/analytic solutions of linear models. $ dMat(\cdot)$ denotes a diagonal matrix, $diag(X)$ is the vector on the diagonal of $X$.}
\label{tab:summary}
\resizebox{\textwidth}{!}{%
\begin{tabular}{|c|l|c|c|}
\hline 
\multicolumn{2}{|c|}{Model}&regularization & solution \\ \hline  
\multirow{8}{*}{Frobenius norm} 
&{\text{1. EASE(full rank) }\citep{Steck_2019}}&
\parbox{8cm}{\begin{equation*} 
\begin{split}
&\min_W{||X-XW||^2_F+\lambda\cdot ||W||^2_F}\\
& s.t. \quad diag(W)=0\\
\end{split}
\end{equation*}} & 
\parbox{6cm}{\begin{equation*}
\begin{split}
C &= (X^TX+\lambda I)^{-1}\\
    W &= I-C\cdot dMat(diag(1\oslash C))\\
    \end{split}
\end{equation*}}\\ \cline{2-4}

&{\text{2. DLAE(full rank) }\citep{DBLP:conf/nips/Steck20}}&
\parbox{8cm}{\begin{equation*} 
\begin{split}
&\min_W{||X-XW||^2_F+ ||\Lambda^{1/2}\cdot W||^2_F}\\
& \Lambda =\frac{p}{1-p}dMat(diag(X^TX))\\
\end{split}
\end{equation*}} & 
\parbox{6cm}{\begin{equation*}
\begin{split}
W = (X^TX+\Lambda)^{-1}X^TX\\
    \end{split}
\end{equation*}}\\ \cline{2-4} 
&{\text{3. EDLAE(full rank) }\citep{DBLP:conf/nips/Steck20}}&
\parbox{8cm}{\begin{equation*} 
\begin{split}
&\min_W{||X-XW||^2_F+ ||\Lambda^{1/2}\cdot W||^2_F}\\
&\Lambda =\frac{p}{1-p}dMat(diag(X^TX))\\
& s.t. \quad diag(W)=0\\
\end{split}
\end{equation*}} & 
\parbox{6cm}{\begin{equation*}
\begin{split}
C &= (X^TX+\Lambda)^{-1}\\
    W &= I-C\cdot dMat(diag(1\oslash C))\\
    \end{split}
\end{equation*}}\\ \cline{2-4} 
&{\text{4. EDLAE-ADMM }\citep{DBLP:conf/nips/Steck20}} &
\parbox{8cm}{\begin{equation*} 
\begin{split}
&\min_{A,B}{||X-XAB^T||^2_F+ ||\Lambda^{1/2}\cdot AB^T||^2_F}\\
& s.t. \quad diag(W)=0\\
\end{split}
\end{equation*}} & 
\parbox{6cm}{\begin{equation*}
\begin{split}
\text{ADMM update } A, B
    \end{split}
\end{equation*}}\\ \cline{2-4} 
&{\text{5. LRR }\citep{JinKDD2021}}&

\parbox{8cm}{\begin{equation*} 
\begin{split}
\min_{rank(W)\le k}||X-XW||^2_F+||\Gamma W||^2_F\\
\end{split}
\end{equation*}} & 
\parbox{6cm}{\begin{equation*}
\begin{split}
\overline{Y}^* &= \overline{X}W^*\stackrel{\text{SVD}}{=}U\Sigma V\\
    \widehat{W} &= (X^TX+\Gamma^T\Gamma)^{-1}X^TX(V_kV^T_k)\\
    \end{split}
\end{equation*}}\\ \cline{2-4}

&{\text{6. LR-DLAE(this paper) }}&
\parbox{8cm}{\begin{equation*} 
\begin{split}
&\min_{rank(W)\le k}{||X-XW||^2_F+ ||\Lambda^{1/2}\cdot W||^2_F}\\
& \Lambda =\frac{p}{1-p}dMat(diag(X^TX))\\
\end{split}
\end{equation*}} & 
\parbox{6cm}{\begin{equation*}
\begin{split}
W^* &= (X^TX+\Lambda)^{-1}X^TX\\
\overline{Y}^* &= \overline{X}W^*\stackrel{\text{SVD}}{=}U\Sigma V^T\\
\widehat{W}&=W^*(V_kV^T_k)\\
    \end{split}
\end{equation*}}\\ \cline{2-4}

&{\text{7. LR-EDLAE-1(this paper) }}&
\parbox{8cm}{\begin{equation*} 
\begin{split}
&\min_{rank(W)\le k}{||X-XW||^2_F+ ||\Lambda^{1/2}\cdot W||^2_F}\\
&\Lambda =\frac{p}{1-p}dMat(diag(X^TX))\\
& s.t. \quad diag(W)=0\\
\end{split}
\end{equation*}} & 
\parbox{6cm}{\begin{equation*}
\begin{split}
C &= (X^TX+\Lambda)^{-1}\\
    W^* &= I-C\cdot dMat(diag(1\oslash C))\\
    \overline{Y}^* &= \overline{X}W^*\stackrel{\text{SVD}}{=}U\Sigma V^T\\
\widehat{W}&=W^*(V_kV^T_k)\\
    \end{split}
\end{equation*}}\\ \cline{2-4} 

&{\text{8. LR-EDLAE-2(this paper) }}&
\parbox{8cm}{\begin{equation*} 
\begin{split}
&\min_{rank(W)\le k}{||X-XW||^2_F+ ||\Lambda^{1/2}\cdot W||^2_F}\\
&\Lambda =\frac{p}{1-p}dMat(diag(X^TX))\\
& \text{s.t.} \quad diag(W)=0\\
\end{split}
\end{equation*}} & 
\parbox{6cm}{\begin{equation*}
\begin{split}
C &= (X^TX+\Lambda)^{-1}\\
    W^* &= I-C\cdot dMat(diag(1\oslash C))\\
W^*&\stackrel{\text{SVD}}{=}U\Sigma V^T\\
\widehat{W}&=U_k\Sigma_k V^T_k
    \end{split}
\end{equation*}}\\

\hline
\multirow{4}{*}{Nuclear Norm} & {\text{9. Regularized PCA }\citep{zheng2018regularized}}&
\parbox{8cm}{\begin{equation*} 
\begin{split}
&\min_{P, Q}{||X-PQ^T||^2_F+\lambda\cdot( ||P||^2_F + ||Q||^2_F)}\\
&X \stackrel{\text{SVD}}{=} U\Sigma V^T\\
\end{split}
\end{equation*}} & 
\parbox{6cm}{\begin{equation*}
\begin{split}
    P^* &= U_k\\
    Q^*&=V_k \Omega \\
    \Omega &= \sqrt{(\sigma_i-\lambda)_{+}}
    \end{split}
\end{equation*}}\\ \cline{2-4} 

&{\text{10. MF dropout }\citep{dropMF}}&
\parbox{8cm}{\begin{equation*} 
\begin{split}
&\min_{P, Q, d}{||X-PQ^T||^2_F+d\frac{1-p}{p}\cdot\sum\limits_{k=1}^d||P_k||^2_2 \cdot ||Q_k||^2_2}\\
&\min_{Y}{||X-Y||^2_F}+\frac{1-p}{p}||Y||^2_*\\
\end{split}
\end{equation*}} & 
\parbox{6cm}{\begin{equation*}
\begin{split}
X &\stackrel{\text{SVD}}{=} U\Sigma V^T\\
Y^* &= P^*\cdot (Q^*)^T\\
&=U\cdot S_\mu(\Sigma) \cdot V^T
    \end{split}
\end{equation*}}\\ \cline{2-4} 
&{\text{11. LAE }\citep{DBLP:conf/nips/BaoLSG20}}&

\parbox{8cm}{\begin{equation*} 
\begin{split}
\min_{W_1, W_2}\|X- X W_1 W_2\|^2_F + \|W_1\Lambda^{\frac 1 2}\|^2_F+ \|\Lambda^{\frac 1 2} W_2\|_F^2,   
\end{split}
\end{equation*}} & 
\parbox{6cm}{\begin{equation*}
\begin{split}
W_1^*&=P(I-\Lambda S^{-2})^{\frac{1}{2}}U^T\\
W_2^*&=U(I-\Lambda S^{-2})^{\frac{1}{2}}P^T
    \end{split}
\end{equation*}}\\\cline{2-4} 

&{\text{12. LVAE(this paper)}}&

\parbox{8cm}{\begin{equation*} 
\begin{split}
&\min_{P,Q}||X-PQ||^2_F+||\Lambda^{1/2}Q||^2_F+||P\Lambda^{1/2}||^2_F\\
&\min_{A,B}||X-XAB||^2_F+||\Lambda B||^2_F+||XA||^2_F\\
&\min_{rank(W)\le k}||X-W||^2_F+2||W||_{w,*}
\end{split}
\end{equation*}} & 
\parbox{8cm}{\begin{equation*}
\begin{split}
&X \stackrel{\text{SVD}}{=} U\Sigma V^T\\
P^*&=U_k \cdot diag(\sqrt{\sigma_1-\lambda_{(k)}},\dots, \sqrt{\sigma_1-\lambda_{(1)}})\cdot \Omega\\
Q^*&=\Omega^T\cdot diag(\sqrt{\sigma_1-\lambda_{(k)}},\dots, \sqrt{\sigma_1-\lambda_{(1)}})\cdot V_k^T\\
A^*&=X^\dagger P^*\Lambda^{\frac{1}{2}}\quad B^*=\Lambda^{-\frac{1}{2}}Q^*
    \end{split}
\end{equation*}}\\ 
\hline
\end{tabular}
}
\end{table}

\section{Proofs} 

\subsection{Proof of Lemma 1}

The Linear Variational AutoEncoder (LVAE) is defined in the same way as \citep{ppca_elbo}:

\begin{equation}
    \begin{split}
        p( {x} \mid  {z}) &=\mathcal{N}\left(W  {z}+\boldsymbol{\mu}, \sigma^{2} I\right)\\
        q( {z} \mid  {x})&=\mathcal{N}(V( {x}-\boldsymbol{\mu}), D)
    \end{split}
\end{equation}

For simplification, we set $\mu = 0$ in following context. And the ELBO of LVAE is known as:
\begin{equation}
\begin{split}
       &\mathcal{L}_x =  -KL(q(z|x)||p(z)) + \mathbb{E}_{q(z|x)} [\log p(x|z)]\\
       &KL(q(z|x)||p(z))=-\log | {D}| +  {x}^T {V}^T {V} {x} + tr( {D}) - k \\
       &\mathbb{E}_{q(z|x)}[\log p(x|z)]=-\frac{1}{2\sigma^2}\Big( tr( {WD} {W}^T) + {x}^T {V}^T {W}^T {W} {V} {x} -2 {x}^T {WVx} +  {x}^T {x}\Big) -\frac{n}{2}\log 2\pi\sigma^2
       \end{split}
\end{equation}

Again, the (maximizing) ELBO can be written as:

\begin{equation}
    \begin{split}
        \mathcal{L}_{ {x}} &= -\frac{1}{2}\Big(-\log | {D}| +  {x}^T {V}^T {V} {x} + tr( {D}) - k \Big) -\frac{n}{2}\log 2\pi\sigma^2\\
        &-\frac{1}{2\sigma^2}\Big( tr( {WD} {W}^T) + {x}^T {V}^T {W}^T {W} {V} {x} -2 {x}^T {WVx} +  {x}^T {x}\Big)\\
        &=-\frac{1}{2}|| {Vx}||^2_2-\frac{1}{2\sigma^2}\Big(|| {W}\sqrt{ {D}}||_F^2 +|| {x-WVx}||_2^2\Big) + f( {D},\sigma)
    \end{split}
\end{equation}
where $f( {D},\sigma) = \frac{1}{2}\log | {D}|- \frac{1}{2}tr( {D}) + \frac{k}{2} -\frac{n}{2}\log 2\pi\sigma^2$, $ {x}\in \mathbb{R}^n$ and $ {z}\in \mathbb{R}^k$.

For whole data, it is equivalent to minimize:
\begin{equation}\label{eq:lvae-james-1}
    \begin{split}
        \mathcal{L} &= || {X-WVX}||_F^2 + N|| {W}\sqrt{ {D}}||_F^2 + \sigma^2|| {VX}||^2_F + g( {D},\sigma)\\
        &= || {X}^T- {X}^T {V}^T {W}^T||_F^2 + N||\sqrt{ {D}} {W}^T||_F^2 + \sigma^2|| {X}^T {V}^T||^2_F + g( {D},\sigma)
    \end{split}
\end{equation}

where $g( {D},\sigma) = -\sigma^2N\big(\log | {D}|- tr( {D}) + k - n\log 2\pi\sigma^2)$. 

\subsection{Proof of Proposition 1}
Note that when $\sigma_{i} \leq \lambda_{(k - i)}$, the new singular value shrinks to zero, and can be removed. 
Basically, for any $\lambda_{(1)}  \geq \dots \geq \lambda_{(k)}$, we can build the corresponding Tikhonov regularized instance by setting
\begin{equation}
    \begin{split}
        \frac{\sigma^2_i}{\sigma^2_i + \lambda^\prime_i} =\frac{\sigma_i - \lambda_{(k - i)}}{\sigma_i}, {\mbox{i.e.,}}  
        \lambda^\prime_k= \frac{\sigma^3_i}{\sigma_i - \lambda_{(k - i)}} - \sigma^2_i.
    \end{split}
\end{equation}

\noindent{\bf Discussion of Proposition 1:}
Further, the same observation holds true for the regularization (\ref{eq:nonuniform}), and the weighted-nuclear norm regularization in when the weights are in the non-ascending order. 
This observation suggests a potentially  limitation of the earlier regularization as they will always try to maintain the larger singular values: when a singular value is large, the shrinkage will be small.Such regularization has shown to work well in the areas such as image processing~\cite{gu2014weighted}. But it has not been studied or confirmed if it will work for the recommendation. 
In Section 5, we report our experimental study which shows such regularization could be too restrictive for recommendation.

\subsection{Proof of Proposition~\ref{prop:diagonalcost}}
By slightly abusing the notation, we shall let $OPT1$ ($OPT2$) be the value of the optimal solution for $OPT1$ ($OPT2$). We need to show that $OPT1 = OPT2$. We need two directions. 

\mypara{$OPT2 \geq OPT1$:} Let $W^*$ be an optimal solution for $OPT2$. Let the SVD of $W^*$ be $U^*\Sigma^* (V^*)^{\transpose}$. Recall that $W^*$ needs to satisfy the rank constraint $\rank(W^*) \leq k$ so $U^* \in \reals^{m \times k}$, $\Sigma^* \in \reals^{k \times k}$, and $V^* \in \reals^{n \times k}$. Let $\pi$ be a permutation on $[k]$ such that $\lambda_{\pi(1)} \leq \lambda_{\pi(2)} \leq \dots \leq \lambda_{\pi(k)}$. Let also $\Omega$ be the corresponding permutation matrix. Specifically, $\Omega \in \{0, 1\}^{k \times k}$ and there is exactly one entry in each row of $\Omega$ is 1: 
\begin{align*}
    \Omega_{i,j} = \left\{
    \begin{array}{ll}
    1     & \mbox{if $j = \pi(i)$.}  \\
    0     & \mbox{otherwise.}
    \end{array}
    \right.
\end{align*}
For example, consider a case in which $\lambda_1 > \lambda_2 > \dots > \lambda_k$. Then we set $\pi = (k, k-1, \dots, 1)$, and correspondingly, 

\begin{align*}
    \Omega = \left(
    \begin{array}{cccc}
    0 & ... & 0 & 1\\
    0 & ... & 1 & 0 \\
    & ... & \\
    1 & ... & 0 & 0 
    \end{array}
    \right).
\end{align*}

Next, let $P = U^*(\Sigma^*)^{\frac 1 2}\Omega$ and $Q = \Omega^{\transpose} (\Sigma^*)^{\frac 1 2}(V^*)^{\transpose}$. We have $W^* = PQ$ and $f(W^*) = f(PQ)$. In addition, 
\begin{align*}
    \|\Lambda^{\frac 1 2}Q \|^2_F + \|P \Lambda^{\frac 1 2}\|^2_F = 2\|\Lambda^{\frac 1 2}\Omega (\Sigma^*)^{\frac 1 2}\|^2_F = 2 \sum_{i \leq k}\lambda_{\pi(i)}\sigma_i = 2\|W^*\|_{\omega, *}, 
\end{align*}
where $\sigma_i$ is the $i$-th largest singular value of $W^*$. In other words, we have found a $(P, Q)$ pair such that 
\begin{align*}
    f(PQ) + \|\Lambda^{\frac 1 2}Q\|^2_F + \|P \Lambda^{\frac 1 2}\|^2_F = f(W^*) + 2\|W^*\|_{\omega, *} = OPT2, 
\end{align*}
which shows that $OPT1 \leq OPT2$.

\mypara{$OPT2 \leq OPT1$.} Let $P^*$ and $Q^*$ be an optimal solution for $OPT1$. Let the singular values of $P^*$ be $\sigma_1(P^*) \geq \sigma_2(P^*) \geq \dots \geq \sigma_k(P^*)$ and those of $Q^*$ be $\sigma_1(Q^*) \geq \sigma_2(Q^*) \geq \dots \geq \sigma_k(Q^*)$. Let also $\sigma^*_1 \geq \dots \geq \sigma^*_k$ be the singular values of $P^*Q^*$. 

We shall find a lower bound of $\|\Lambda^{\frac 1 2}Q\|^2_F + \|P\Lambda^{\frac 1 2}\|^2_F$ expressed in terms of $\sigma^*_i$'s. In fact, we shall show that
\begin{equation}\label{eqn:nuclearinq}
    \|\Lambda^{\frac 1 2}Q\|^2_F + \|P\Lambda^{\frac 1 2}\|^2_F \geq 2\|P^*Q^*\|_{\omega, *}. 
\end{equation}
One can see that if Eq.~\ref{eqn:nuclearinq} were true, 
we have 
\begin{align*}
    OPT2 \leq f(P^*Q^*) + 2\|P^*Q^*\|_{\omega, *} \leq f(P^*Q^*) + \|\Lambda^{\frac 1 2}Q^*\|^2_F + \|P^* \Lambda^{\frac 1 2}\|^2_F = OPT1. 
\end{align*}
Thus, it remains to prove Eq.~\ref{eqn:nuclearinq}. Let $\lambda_{(1)} \geq \lambda_{(2)} \geq \dots \geq \lambda_{(k)}$ be a sorted sequence of $\lambda_i$'s i.e., $\lambda_{(k)} = \lambda_{\pi(1)}$, $\lambda_{(k - 1)} = \lambda_{\pi(2)}, \dots, \lambda_{(1)} = \lambda_{\pi(k)}$. 

First, we show that $\|P\Lambda^{\frac 1 2}\|^2_F \geq \sum_{i = 1}^k\lambda_{(k - i + 1)}\times \sigma^2_i(P^*)$ and  $\|\Lambda^{\frac 1 2}Q\|^2_F \geq \sum_{i = 1}^k\lambda_{(k - i + 1)}\times \sigma^2_i(Q^*)$. We need the following Lemma (see e.g., Theorem 2 in~\cite{yue2020matrix}): 

\begin{lemma}\label{lem:traceineq} Let $A$ and $B$ be two positive definite matrices in $\reals^{k \times k}$. Then it holds that 
\begin{equation}
    \sum_{i = 1}^k \sigma_i(A) \sigma_{k - i + 1}(B) \leq \trace(B^{\frac 1 2}A B^{\frac 1 2}). 
\end{equation}
\end{lemma}

Let the SVD of $P^*$ be $U_{P^*} \Sigma_{P^*} V^{\transpose}_{P^*}$ and that of $Q^*$ be $U_{Q^*} \Sigma_{Q^*} V^{\transpose}_{Q^*}$. We have 
\begin{equation}
    \|P^*\Lambda^{\frac 1 2}\|^2_F = \|U_{P^*} \Sigma_{P^*}V^{\transpose}_{P^*}\Lambda^{\frac 1 2} \|^2_F = \|\Sigma_{P^*}V^{\transpose}_{P^*}\Lambda^{\frac 1 2} \|^2_F = 
     \trace(\Sigma_{P^*}V^{\transpose}_{P^*}\Lambda V_{P^*}\Sigma_{P^*}). 
\end{equation}
We now apply Lemma~\ref{lem:traceineq} by setting $A = V^{\transpose}_{P^*}\Lambda V_{P^*}$ and $B = \Sigma^2_{P^*}$, and obtain that 
\begin{equation}
    \|P^*\Lambda^{\frac 1 2}\|^2_F  = \trace(\Sigma_{P^*}V^{\transpose}_{P^*}\Lambda V_{P^*}\Sigma_{P^*}) \geq \sum_{i = 1}^k\lambda_{(k + 1 - i)} \times \sigma^2_i(P^*). 
\end{equation}
We may similarly prove that $\|\Lambda^{\frac 1 2}Q^*\|^2_F \geq \sum_{i = 1}^k\lambda_{(k - i + 1)}\times \sigma^2_i(Q^*)$. Therefore, 
\begin{equation}
    \|\Lambda^{\frac 1 2}Q^*\|^2_F + \|P^*\Lambda^{\frac 1 2}\|^2_F \geq \sum_{i = 1}^k \lambda_{(k + 1 - i)}\times (\sigma^2_i(P^*) + \sigma^2_i(Q^*))
    \label{eqn:bound1}
\end{equation}

(\ref{eqn:bound1}) provides a lower bound of $\|\Lambda^{\frac 1 2} Q^*\|^2_F + \|P^*\Lambda^{\frac 1 2}\|^2_F$ in terms of $\sigma_i(P^*)$ and $\sigma_i(Q^*)$. We next aim to express the lower bound in terms of $\sigma^*_i$'s (singular values of $P^*Q^*$) directly. 

The following program gives a lower bound for $\|\Lambda^{\frac 1 2} Q^*\|^2_F + \|P^*\Lambda^{\frac 1 2}\|^2_F$: 
\begin{align}
    \min: \quad & \|\Lambda^{\frac 1 2} Q^*\|^2_F + \|P^*\Lambda^{\frac 1 2}\|^2_F \label{eqn:minpq} \\
    \mbox{subject to} \quad & W = P^*Q^* \nonumber\\
    & \sigma_i(W) = \sigma^*_i \quad {\mbox{for $i \leq k$}}. \nonumber
\end{align}
Write the SVD of $W$ be $U_W\Sigma_W V^{\transpose}_W$. Also, let $\tilde P = U^{\transpose}_W P^*$ and $\tilde Q = Q^* V_W$. 
Noting that the columns in $P^*$ are in the column space of $W$ and the rows in $Q^*$ are in the row space of $W$, we have \emph{(i)} $\sigma_i(P^*) = \sigma_i(\tilde P)$ and $\sigma_i(Q^*) = \sigma_i(\tilde Q)$ for $i \leq k$, and \emph{(ii)}
$\|\Lambda^{\frac 1 2} Q^*\|^2_F + \|P^*\Lambda^{\frac 1 2}\|^2_F = \|\Lambda^{\frac 1 2} \tilde Q\|^2_F + \|\tilde P\Lambda^{\frac 1 2}\|^2_F$. 

Therefore, (\ref{eqn:minpq}) can be equivalently written as 
\begin{align}
    \min: \quad & \|\Lambda^{\frac 1 2} \tilde Q\|^2_F + \|\tilde P\Lambda^{\frac 1 2}\|^2_F \label{eqn:minpqe} \\
    \mbox{subject to} \quad & \Sigma_W = \tilde P\tilde Q \nonumber\\
    & (\Sigma_{W})_{i,i} = \sigma^*_i \quad {\mbox{for $i \leq k$}}. \nonumber
\end{align}
Now $\tilde P\tilde Q$ is positive definite. Using a similar technique developed in~\citep{DBLP:conf/nips/BaoLSG20} (Theorem 1), one can see that $\tilde P = \tilde Q^{\transpose}$. See also Lemma~\ref{lem:lastminpproof}. This implies that $\tilde P = \Sigma^{\frac 1 2}_W\Omega$ for some unitary matrix $\Omega$ and $\sigma_i(P^*) = \sigma_i(\tilde P) = \sigma_i(Q^*) = \sigma_i(\tilde Q) = \sqrt{\sigma^*_i}$ for $i \leq k$. Together with (\ref{eqn:bound1}), we have
\begin{align*}
      \|\Lambda^{\frac 1 2}Q^*\|^2_F + \|P^*\Lambda^{\frac 1 2}\|^2_F \geq 2 \sum_{i \leq k}\lambda_{(k + 1 - i)} \times \sigma^2_i(P^*) = 2 \sum_{i \leq k}\lambda_{(k + 1 - i)} \times \sigma^*_i= 2\|P^*Q^*\|_{\omega, *}. 
\end{align*}

\subsection{Proof of Corollary~\ref{cor:closedform}}

We first find an optimal solution for (\ref{eqn:factorlambda}). Let the SVD of $X$ be $X = U_X\Sigma_X V^{\transpose}_X$, where $U_X \in \reals^{m \times n}$, $\Sigma_X \in \reals^{n \times n}$, and $V_X \in \reals^{n \times n}$. Let $\bar U_X$ be an arbitrary basis for the subspace that is orthogonal to $X$'s column space so $\bar U_X \in \reals^{m \times (m - n)}$ and $[U_X, \bar U_X]$ form a basis for $\reals^{m}$. We have 
\begin{align*}
    & \|X - PQ\|^2_F + \|P \Lambda^{\frac 1 2}\|^2_F + \|\Lambda^{\frac 1 2}Q\|^2_F \\
=   &  \left\|
    \left(\begin{array}{c}
         U^{\transpose}_X  \\
         \bar U^{\transpose}_X 
    \end{array}\right) X V_X - \left(\begin{array}{c}
         U^{\transpose}_X  \\
         \bar U^{\transpose}_X 
    \end{array}\right) P Q V_X
    \right\|^2_F + \left\|\left(\begin{array}{c}
         U^{\transpose}_X  \\
         \bar U^{\transpose}_X 
    \end{array}\right)P\Lambda^{\frac 1 2}\right\|^2_F + \|\Lambda^{\frac 1 2}Q V_X\|^2_F. 
\end{align*}

Let $\tilde P = \left(\begin{array}{c}
         U^{\transpose}_X  \\
         \bar U^{\transpose}_X 
    \end{array}\right) P$ and $\tilde Q = QV_X$. Then our objective becomes 
    
\begin{align}
    \min_{\tilde P, \tilde Q} \left\|\left(
    \begin{array}{c}
    \Sigma_X \\
    0_{(m - n) \times n}
    \end{array}
    \right) - \tilde P \tilde Q\right\|^2_F + \|\tilde P \Lambda^{\frac 1 2}\|^2_F + \|\Lambda^{\frac 1 2}\tilde Q\|^2_F. 
\end{align}
    
Let $\tilde W = \tilde P \tilde Q$ and the singular values of $\tilde W$ be $\sigma^*_1 \geq \sigma^*_2 \geq \dots \geq \sigma^*_k$. Let also $\tilde \Sigma = 
\left(
\begin{array}{c}
     \Sigma_X  \\
      0_{(m - n) \times n}
\end{array}
\right)$. 

Recall also that $\sigma_i$ is the $i$-th largest singular value of $X$. We next show that 
\begin{align*}
    \left\|\left(
    \begin{array}{c}
    \Sigma_X \\
    0
    \end{array}
    \right) - \tilde P \tilde Q\right\|^2_F = \|\tilde \Sigma - \tilde P \tilde Q \|^2_F \geq \sum_{i = 1}^k(\sigma_i - \sigma^*_i)^2 + \sum_{i = k + 1}^n \sigma^2_i. 
\end{align*}

Note first that 
\begin{align}
    \|\tilde \Sigma - \tilde W\|^2_F = \|\tilde \Sigma\|^2_F + \|\tilde W\|^2_F - 2 \langle \tilde \Sigma, \tilde W\rangle. \label{eqn:intermediate}
\end{align}
Next, we have (\citep{zheng2018regularized}): 
\begin{align*}
    |\langle \tilde \Sigma, \tilde W\rangle | = | \trace(\tilde \Sigma \tilde W^{\transpose})\| 
    \leq |\trace(\tilde \Sigma \Sigma_{\tilde W})| = \sum_{i = 1}^k\sigma_i \sigma^*_i. 
\end{align*}
Therefore, $\langle \tilde \Sigma, \tilde W\rangle$ is maximized when 
\begin{align*}
    \tilde W_{i, j} = \left\{
    \begin{array}{cc}
    \sigma^*_i     &  \mbox{if $i = j \leq k$} \\
    0 &  \mbox{Otherwise.}
    \end{array}
    \right.
\end{align*}
When we plug in this optimized $\tilde W$ to Eq.~\ref{eqn:intermediate}, we get 
\begin{align*}
    \|\tilde \Sigma - \tilde W\|^2_F \geq \sum_{i = 1}^k(\sigma_i - \sigma^*_i)^2 + \sum_{i = k + 1}^n \sigma^2_i. 
\end{align*}

Next, from Proposition~\ref{prop:diagonalcost}, we have 
\begin{align*}
    \|\tilde P V^{\frac 1 2}\|^2_F + \|\Lambda^{\frac 1 2}\tilde Q\|^2_F \geq \sum_{i = 1}^k \lambda_{(k - i + 1)}\sigma^*_i. 
\end{align*}

Therefore, we can find a lower bound for Eq.~\ref{eqn:optclosedform} in terms of $\sigma^*_i$'s: 

\begin{align}
    \calL(\sigma^*_1, \dots, \sigma^*_k) = \sum_{i = 1}^k (\sigma_i - \sigma^*_i)^2 + 2 \sum_{i = 1}^k \lambda_{(k - i + 1)}\sigma^*_i + \sum_{i = k + 1}^m\sigma^2_i \quad (\sigma^*_1 \geq \dots \geq \sigma^*_k \geq 0). 
\end{align}

We next find a minimal value of $\calL$ (by treating $\sigma^*_i$'s as decision variables). This will give us a lower bound (and is independent of $\sigma^*_i$) on our optimization problem. We then show that this lower bound can be achieved by carefully constructing $\tilde W$ (as well as $\tilde P$ and $\tilde Q$). This means such $\tilde W$ is optimal. 

Specifically, we need to find an optimal solution for the following program:

\begin{align}
    \mbox{minimize}_{\sigma^*_1, \dots, \sigma^*_k} \quad & \calL(\sigma^*_1, \dots, \sigma^*_k) \label{eqn:constraintsigma}\\
    \mbox{subject to: } \quad  & \sigma^*_i \geq 0  \nonumber \\
    & \sigma^*_1 \leq \sigma^*_2 \leq \dots \leq \sigma^*_k \quad {\mbox{(Ordering constraint)}} \nonumber 
\end{align}

We shall first find an optimal solution for

\begin{align}
    \mbox{minimize}_{\sigma^*_1, \dots, \sigma^*_k} \quad & \calL(\sigma^*_1, \dots, \sigma^*_k) \label{eqn:constraintsigma_droporder}\\
    \mbox{subject to: } \quad  & \sigma^*_i \geq 0  \nonumber 
\end{align}

Note here, the ordering constraint is removed so the optimal value for (\ref{eqn:constraintsigma_droporder}) should be no more than that for (\ref{eqn:constraintsigma}). We shall see that the optimal solution for (\ref{eqn:constraintsigma}) also satisfies the ordering constraint so indeed optimal solutions for (\ref{eqn:constraintsigma}) and (\ref{eqn:constraintsigma_droporder}) are the same. 

The problem (\ref{eqn:constraintsigma_droporder}) boils down to finding 
\begin{align*}
    \min_{\sigma^*_i \geq 0} (\sigma_i - \sigma^*_i)^2 + 2\sum_{i = 1}^k \lambda_{(k - i + 1)}\sigma^*_i. 
\end{align*}

We note that $\sigma^*_i$'s do not interact with each other so we can optimize each $\sigma^*_i$'s independently. We get 
\begin{align*}
    \sigma^*_i = (\sigma_i - \lambda_{(k - i + 1)})^+. 
\end{align*}
We can check that $\sigma^*_1 \geq \dots \geq \sigma^*_k$. Therefore, the optimal value for (\ref{eqn:constraintsigma}) is 
\begin{align*}
    \sum_{i = 1}^k (\sigma_i - (\sigma_i - \lambda_{(k  -  i +1)})^+)^2 + 2 \sum_{i = 1}^k \lambda_{(k - i + 1)}(\sigma_i - \lambda_{(k - i + 1)})^+ + \sum_{i = k + 1}^n\sigma^2_i. 
\end{align*}
This is also a lower bound for (\ref{eqn:factorlambda}). One can check that when we set $P$ and $Q$ as 

\begin{align}
    P^* & = U_k diag(\sqrt{(\sigma_1 - \lambda_{(k)})^+}), \dots, \sqrt{(\sigma_k - \lambda_{(1)})^+}) \Omega,\label{eqn:optimalpq}  \\
    Q^* & = \Omega^{\transpose}diag(\sqrt{(\sigma_1 - \lambda_{(k)})^+}), \dots, \sqrt{(\sigma_k - \lambda_{(1)})^+}) V^{\transpose}_k, \nonumber
\end{align}

the lower bound is achieved so (\ref{eqn:optimalpq}) gives an optimal solution. Here, $U_k$ and $V_k$ are leading left and right singular vectors of $X$. 

Now we move to analyze (\ref{eqn:factorlambdaab}). Our goal is to reduce (\ref{eqn:factorlambdaab}) to (\ref{eqn:factorlambda}). Let 
\begin{align*}
    P = X A \Lambda^{-\frac 1 2} \quad \quad \quad Q = \Lambda^{\frac 1 2}B. 
\end{align*}
Then (\ref{eqn:factorlambdaab}) becomes 
\begin{align}
    \mbox{minimize}_{P, Q} \quad & \|X - PQ\|^2_F + \|P \Lambda^{\frac 1 2}\|^2_F + \|\Lambda^{\frac 1 2}Q\|^2_F
     \label{eqn:solveoptimalab}\\
    \mbox{subject to} \quad & P = X A\Lambda^{-\frac 1 2} \quad \mbox{(Constraint P)} \nonumber \\
    & Q = \Lambda^{\frac 1 2}B \quad \mbox{(Constraint Q)}. \nonumber 
\end{align}
Here, $X$ and $\Lambda$ are given, whereas $P$, $Q$, $A$, and $B$ are decision variables. The (Constraint P) says that each column of $P$ needs to be in a column space of $X$ (it is a necessary and sufficient condition for $A$ to exist).  The (Constraint Q) simply says $Q$ and $B$ are linearly related and does not have tangible impact to the optimization problem. 

But we note that when we put aside the constraints, an optimal $(P, Q)$ is specified by (\ref{eqn:optimalpq}). The columns of the optimal $P$ indeed is in the column space of $X$. So $(P, Q)$ is also an optimal solution for (\ref{eqn:solveoptimalab}). We may find the corresponding $A$ and $B$: 

\begin{align*}
A^* = X^{\dagger}P^* \Lambda^{\frac 1 2} \quad \mbox{ and } B^* = \Lambda^{- \frac 1 2}Q^*, 
\end{align*}

\subsection{Symmetric lemma}

\begin{lemma}\label{lem:lastminpproof} Let $\tilde P, \tilde Q \in \reals^{k \times k}$ be full rank, $\Lambda$ be a diagonal matrix, and $\Sigma_{W}$ be a diagonal matrix so that $(\Sigma_W)_{i,i} = \sigma^*_i$, where $\sigma^*_i$'s are sorted in descending order. Consider the optimization problem:
\begin{align}
    \min: \quad & \|\Lambda^{\frac 1 2} \tilde Q\|^2_F + \|\tilde P\Lambda^{\frac 1 2}\|^2_F \label{eqn:minpqere} \\
    \mbox{subject to} \quad & \Sigma_W = \tilde P\tilde Q \nonumber\\
    & (\Sigma_{W})_{i,i} = \sigma^*_i \quad {\mbox{for $i \leq k$}}. \nonumber
\end{align}
There is an optimal solution such that $\tilde P = \tilde Q^{\transpose}$
\end{lemma}
\begin{proof} Let $\hat P = \tilde P \Lambda^{\frac 1 2}$ and $\hat Q = \Lambda^{-\frac 1 2}\tilde Q$. The program (\ref{eqn:minpqere}) is equivalent to 

\begin{align}
    \min: \quad & \|\Lambda \hat Q\|^2_F + \|\hat P\|^2_F \label{eqn:minpqere} \\
    \mbox{subject to} \quad & \Sigma_W = \hat P\hat Q \nonumber\\
    & (\Sigma_{W})_{i,i} = \sigma^*_i \quad {\mbox{for $i \leq k$}}. \nonumber
\end{align}
Let the SVD of $\hat Q$ be $U_{\hat Q}\Sigma_{\hat Q}V^{\transpose}_{\hat Q}$ so $\hat Q^{-1} = V_{\hat Q}\Sigma^{-1}_{\hat Q}U^{\transpose}_{\hat Q}$. We can also see that $\hat P =\Sigma_W \hat Q^{-1}$. Therefore, the objective term becomes
\begin{align*}
    \|\Lambda U_{\hat Q}\Sigma_{\hat Q}V^{\transpose}_{\hat Q}\|^2_F + \|\Sigma_{W}V_{\hat Q} \Sigma^{-1}_{\hat Q}U^{\transpose}_{\hat Q}\|^2_F = \|\Lambda U_{\hat Q}\Sigma_{\hat Q}\|^2_F + \|\Sigma_W V_{\hat Q} \Sigma^{-1}_{\hat Q}\|^2_F. 
\end{align*}

Let us consider the stationary points $U_{\hat Q}$ and $V_{\hat Q}$ when $\Sigma_{\hat Q}$ is fixed. We can see that they need to be permutation matrices to minimize both terms in the objective (using the rearrangement inequality again). Therefore, we can see $\hat Q = \Sigma_1 \Sigma_{\hat Q}\Sigma_2$ for two permutation matrices $\Sigma_1$ and $\Sigma_2$. This implies that $\tilde Q = \Lambda^{\frac 1 2}\Sigma_1 \Sigma_{\hat Q}\Sigma_2$, i.e., each row (column) of $\tilde Q$ has exactly one non-zero entry. We may similarly show that each row (column) of $\tilde P$ has exactly one non-zero entry. In addition, the locations of non-zero entries of $\tilde P$ and $\tilde Q^{\transpose}$ are identical because $\tilde P \tilde Q$ is a diagonal matrix. We may thus write
\begin{align*}
    \tilde P = \Sigma_{(1)} \Sigma_{\tilde P}\Sigma_{(2)} \quad \tilde Q = \Sigma^{\transpose}_{(2)} \Sigma_{(\tilde Q)}\Sigma^{\transpose}_{(1)}, 
\end{align*}
where $(\Sigma_{\tilde P})_{i,i} = \sigma_i(\tilde P)$ and $(\Sigma_{(\tilde Q)})_{i,i} = \sigma_{\tau(i)}(\tilde Q)$, where $\tau$ is a permutation on $[k]$. The set of (possibly unsorted) singular values for $\tilde P \tilde Q$ thus is $\sigma_i(\tilde P) \sigma_{\tau(i)}(\tilde Q)$. 
Thus, we can see that there exists a permutation $\bar \pi$ such that 
\begin{align*}
    \|\Lambda^{\frac 1 2}\tilde Q\|^2_F +  \|\tilde P\Lambda^{\frac 1 2}\|^2_F & = \sum_{i \leq k}\big(\sigma^2_i(P^*)\lambda_{\bar \pi(i)} + \sigma^2_{\tau(i)}(Q^*)\lambda_{\bar \pi(i)} \big)\\
    & \geq \sum_{i \leq k}2 \sigma_i(P^*)\sigma^*_{\tau(i)}(Q) \lambda_{\bar \pi(i)} \\
    & \geq 2 \|P^*Q^*\|_{\omega, *}. 
\end{align*}
One can see that we can set $\tilde P = \tilde Q^{\transpose}$ to make all inequality becomes equality so there is an optimal solution such that $\tilde P = \tilde Q^{\transpose}$. 
\end{proof}

\section{Experimental Details}

\subsection{DLAE hyperparameters tuning}

This section presents the hyperparameter tuning process on the validation data over three (ML-20M, Netflix, MSD) datasets for the full rank DLAE formula, which was introduced by \citet{DBLP:conf/nips/Steck20} yet not investigated:
\begin{equation*}
    \begin{split}
        &\min_W{||X-XW||^2_F+ ||\Lambda^{1/2}\cdot W||^2_F}\\
& \Lambda =\frac{p}{1-p}dMat(diag(X^TX))\\
&\hat{W} = (X^TX+\Lambda)^{-1}X^TX\\
    \end{split}
\end{equation*}
In practical, $l_2$ regularization is also imposed:
\begin{equation*}
    \begin{split}
\hat{W} = (X^TX+\Lambda + \lambda)^{-1}X^TX\\
    \end{split}
\end{equation*}
The \cref{tab:dlae_full_ml20m,tab:dlae_full_netflix,tab:dlae_full_msd} show the results of $nDCG@100$ over three datasets respectively. And the optimal parameters are highlighted.

\begin{table}[h]
\centering
\caption{ml-20m, DLAE full rank, parameter tuning on validation dataset by $nDCG@100$}
\label{tab:dlae_full_ml20m}
\begin{small}
\begin{tabular}{|c|c|c|c|c|c|c|c|}
\hline
\multicolumn{2}{|l|}{\multirow{2}{*}{}} & \multicolumn{6}{c|}{$\lambda$}                                     \\ \cline{3-8} 
\multicolumn{2}{|l|}{}                  & 800     & 900     & \textbf{1000}    & 1100    & 1200    & 1300    \\ \hline
\multirow{5}{*}{p}    & 0.1             & 0.42024 & 0.42063 & 0.42073          & 0.42102 & 0.42131 & 0.4212  \\ \cline{2-8} 
                      & 0.2             & 0.43132 & 0.43139 & 0.43154          & 0.4314  & 0.43147 & 0.43136 \\ \cline{2-8} 
                      & \textbf{0.3}    & 0.43203 & 0.43211 & \textbf{0.43214} & 0.43206 & 0.43203 & 0.43196 \\ \cline{2-8} 
                      & 0.4             & 0.43001 & 0.43001 & 0.42995          & 0.42996 & 0.42984 & 0.42978 \\ \cline{2-8} 
                      & 0.5             & 0.42754 & 0.42745 & 0.42729          & 0.42718 & 0.42715 & 0.42704 \\ \hline
\end{tabular}%
\end{small}
\end{table}

\begin{table}[h]
\centering
\caption{netflix, DLAE full rank, parameter tuning on validation dataset by $nDCG@100$}
\label{tab:dlae_full_netflix}
\begin{small}
\begin{tabular}{|c|c|c|c|c|c|c|c|c|}
\hline
\multicolumn{2}{|l|}{\multirow{2}{*}{}} & \multicolumn{7}{c|}{$\lambda$}                                               \\ \cline{3-9} 
\multicolumn{2}{|l|}{}                  & 800     & 900     & 1000    & \textbf{1100}    & 1200    & 1300    & 1400    \\ \hline
\multirow{7}{*}{p}    & 0.2             & 0.3904  & 0.3904  & 0.39027 & 0.39024          & 0.3902  & 0.3903  & 0.39018 \\ \cline{2-9} 
                      & 0.25            & 0.39247 & 0.39252 & 0.39248 & 0.39249          & 0.3925  & 0.3925  & 0.39256 \\ \cline{2-9} 
                      & 0.3             & 0.39359 & 0.39359 & 0.39366 & 0.39358          & 0.39362 & 0.39368 & 0.39369 \\ \cline{2-9} 
                      & \textbf{0.35}   & 0.39402 & 0.39403 & 0.394   & \textbf{0.39405} & 0.39399 & 0.39403 & 0.39397 \\ \cline{2-9} 
                      & 0.4             & 0.39399 & 0.39393 & 0.39395 & 0.39393          & 0.39389 & 0.39388 & 0.39387 \\ \cline{2-9} 
                      & 0.45            & 0.39346 & 0.3935  & 0.39343 & 0.39344          & 0.39338 & 0.3933  & 0.39329 \\ \cline{2-9} 
                      & 0.5             & 0.39249 & 0.39241 & 0.39247 & 0.39241          & 0.39242 & 0.3923  & 0.39224 \\ \hline
\end{tabular}%
\end{small}
\end{table}

\begin{table}[]
\centering
\caption{msd, DLAE full rank, parameter tuning on validation dataset by $nDCG@100$}
\label{tab:dlae_full_msd}
\begin{small}
\begin{tabular}{|c|c|c|c|c|c|c|c|}
\hline
\multicolumn{2}{|l|}{\multirow{2}{*}{}} & \multicolumn{6}{c|}{$\lambda$}                                     \\ \cline{3-8} 
\multicolumn{2}{|l|}{}                  & 10      & 20      & \textbf{30}      & 40      & 50      & 60      \\ \hline
\multirow{4}{*}{p}    & 0.3             & 0.38514 & 0.38515 & 0.38517          & 0.38505 & 0.38492 & 0.38474 \\ \cline{2-8} 
                      & \textbf{0.4}    & 0.38596 & 0.38599 & \textbf{0.38602} & 0.386   & 0.38597 & 0.38592 \\ \cline{2-8} 
                      & 0.5             & 0.38556 & 0.38555 & 0.38553          & 0.38557 & 0.38553 & 0.38549 \\ \cline{2-8} 
                      & 0.6             & 0.38382 & 0.3838  & 0.38381          & 0.38374 & 0.38373 & 0.38366 \\ \hline
\end{tabular}%
\end{small}
\end{table}

\subsection{Matrix Factorization with Dropout Hyperparameters Tuning}
\citet{dropMF} shows that optimization with dropout (allowing rank optimizing) is equivalent to solving a matrix approximation problem with nuclear norm:
\begin{equation*} 
\begin{split}
&\min_{P, Q, d}{||X-PQ^T||^2_F+d\frac{1-p}{p}\cdot\sum\limits_{k=1}^d||P_k||^2_2 \cdot ||Q_k||^2_2}\\
&\min_{Y}{||X-Y||^2_F}+\frac{1-p}{p}||Y||^2_*\\
\end{split}
\end{equation*}

and the solution is given by:
\begin{equation*}
    \begin{split}
        X &\stackrel{\text{SVD}}{=} U\Sigma V^T\\
Y^* &= P^*\cdot (Q^*)^T\\
&=U\cdot S_\mu(\Sigma) \cdot V^T\\
S_\mu(\sigma)&=\max(\sigma - \mu, 0)\\
\mu &=\frac{1-p}{p+(1-p)\bar{d}}\sum\limits_{i=1}^{\bar{d}}{\sigma_i(X)}
    \end{split}
\end{equation*}
where $\bar{d}$ denotes the largest integer such that:
\begin{equation*}
    \sigma_{\bar{d}}(X)>\frac{1-p}{p+(1-p)\bar{d}}\sum\limits_{i=1}^{\bar{d}}{\sigma_i(X)}
\end{equation*}
Hence, there is only one parameter $p$ to tuning. We present the tuning process on the validation set below, see \cref{tab:mf_drop_ml20m,tab:mf_drop_netflix,tab:mf_drop_msd}. Optimal parameters as well as induced rank $\bar{d}$ are highlighted.

\begin{table}[h]
\centering
\caption{ml-20m, matrix factorization with dropout, hyper parameter tuning by $nDCG@100$ on validation dataset and its induced rank .}
\label{tab:mf_drop_ml20m}
\begin{small}
\begin{tabular}{|c|c|c|c|c|c|}
\hline
p        & 0.9     & 0.99    & 0.995   & \textbf{0.996}   & 0.997   \\ \hline
induced rank $\bar{d}$   & 10      & 200     & 385     & \textbf{467}              & 602     \\ \hline
nDCG@100 & 0.29723 & 0.39369 & 0.40045 & \textbf{0.40046} & 0.39925 \\ \hline
\end{tabular}%
\caption{netflix, matrix factorization with dropout, hyper parameter tuning by $nDCG@100$ on validation dataset and its induced rank .}
\label{tab:mf_drop_netflix}
\begin{tabular}{|c|c|c|c|c|c|}
\hline
p        & 0.9     & 0.99    & 0.996   & \textbf{0.997}   & 0.998   \\ \hline
induced rank $\bar{d}$   & 9       & 209     & 524     & \textbf{653}              & 883     \\ \hline
nDCG@100 & 0.26026 & 0.35462 & 0.36453 & \textbf{0.36495} & 0.36406 \\ \hline
\end{tabular}%
\caption{msd, matrix factorization with dropout, hyper parameter tuning by $nDCG@100$ on validation dataset and its induced rank .}
\label{tab:mf_drop_msd}
\begin{tabular}{|c|c|c|c|c|c|}
\hline
p        & 0.99    & 0.999   & 0.9995 & \textbf{0.9999}  & 0.99995 \\ \hline
induced rank $\bar{d}$   & 249     & 2054    & 3783   & \textbf{11380}            & 19308   \\ \hline
nDCG@100 & 0.18986 & 0.28532 & 0.307  & \textbf{0.32634} & 0.30995 \\ \hline
\end{tabular}%
\end{small}
\end{table}

\subsection{Resources}
 Our code are mainly implemented in Numpy 1.19, Pytorch 1.7.1 on CUDA 11.0. Our experiments are performed on nodes with two sockets, each containing a 24-core Intel(R) Xeon(R) Platinum 8268 CPU @ 2.90GHz and 4 GeForce RTX 3090 24GB memory GPU.

\end{document}